\documentclass[11pt]{article}
\usepackage{graphicx}%
\usepackage{multirow}%
\usepackage{amsmath,amssymb,amsfonts}%
\usepackage{amsthm}%
\usepackage{mathrsfs}%
\usepackage[title]{appendix}%
\usepackage{xcolor}%
\usepackage{textcomp}%
\usepackage{manyfoot}%
\usepackage{booktabs}%
\usepackage{algorithm}%
\usepackage{algorithmicx}%
\usepackage{algpseudocode}%
\usepackage{listings}%

\usepackage{caption}
\usepackage{subcaption}
\newtheorem{thm}{Theorem}[section]
\newcommand{\mrd}{\mathrm{d}}

\usepackage[a4paper]{geometry}
\usepackage[square,numbers]{natbib}
\usepackage[colorlinks, linkcolor=red, anchorcolor=blue, citecolor=blue]{hyperref}
\title{An efficient likelihood-free Bayesian inference method based on sequential neural posterior estimation}
\author{Yifei Xiong$^{1}$, Xiliang Yang$^2$, Sanguo Zhang$^{3}$, Zhijian He$^2$\thanks{Corresponding author: hezhijian@scut.edu.cn}}
\date{%
    $^1$Department of Statistics, Purdue University, West Lafayette, United States\\
    $^2$School of Mathematics, South China University of Technology, Guangzhou, China \\
    $^3$School of Mathematical Sciences, University of Chinese Academy of Sciences, Beijing, China\\
}

\begin{document}

\maketitle
\begin{abstract}
Sequential neural posterior estimation (SNPE) techniques have been recently proposed for dealing with simulation-based models with intractable likelihoods. Unlike approximate Bayesian computation, SNPE techniques learn the posterior from sequential simulation using neural network-based conditional density estimators by minimizing a specific loss function. 
The SNPE method proposed by Lueckmann et al. (2017) used a calibration kernel to boost the sample weights around the observed data, resulting in a concentrated loss function. However, the use of calibration kernels may increase the variances of both the empirical loss and its gradient, making the training inefficient. 
To improve the stability of SNPE, this paper proposes to use an adaptive calibration kernel and several variance reduction techniques. The proposed method greatly speeds up the process of training and provides a better approximation of the posterior than the original SNPE method and some existing competitors as confirmed by numerical experiments. We also managed to demonstrate the superiority of the proposed method for a high-dimensional model with a real-world dataset.
\end{abstract}

\section{Introduction} \label{section:intro}
Simulator-based models are extensively used in a vast number of applications in science, including neuroscience \citep{paninski2018neural}, physics \citep{brehmer2020mining, gonccalves2020training}, biology \citep{hashemi2023amortized, clarke2008properties}, and inverse graphics \citep{romaszko2017vision}, since it is an important tool to describe and understand the process being investigated from the provided observation. When applying traditional Bayesian inference on a simulator-based model, one finds difficulties like intractable likelihood function $p(x|\theta)$ and running the simulator can be computationally expensive. 

To address these challenges, a series of likelihood-free Bayesian computation (LFBC) methods have been developed in the past 20 years. These methods aim to perform Bayesian inference without explicitly calculating the likelihood function, such as approximate Bayesian computation (ABC) \citep{beaumont2002approximate, marin2012approximate} and synthetic likelihoods (SL) \citep{wood2010statistical, price2018bayesian}, Bayes optimization \citep{gutmann2016bayesian}, likelihood-free inference by ratio estimation \citep{thomas2022likelihood} and pseudo marginal methods \citep{andrieu2009pseudo, andrieu2010particle}. 

Currently, optimization-based approaches are widely employed for LFBC. These approaches aim to minimize the Kullback-Leibler (KL) divergence between the variational density $q_\phi(\theta)$ and the posterior distribution $p(\theta|x_o)$. One prominent optimization-based approach is the variational Bayes (VB) method. In the likelihood-free context, progress has been made in applying VB. \citet{barthelme2014expectation} utilized expectation propagation, a variational approximation algorithm, to approximate posteriors. \citet{he2022unbiased} developed an unbiased VB method based on nested Multi-level Monte Carlo to handle intractable likelihoods. However, these approaches require nested simulations to estimate the gradient estimator, resulting in additional computational costs.

In recent research, there has been a growing interest in combining neural networks with classic VB methods, particularly using neural network architectures like normalizing flows \citep{kobyzev2020normalizing}. The integration of neural networks and VB provides a flexible and powerful framework for density estimation, offering multiple choices for the variational density. This combination has gained attention due to its outstanding performance. Sequential neural likelihood (SNL) method \citep{papamakarios2019sequential} approximates the likelihood $p(x|\theta)$ by a neural network-based density estimator $q_\phi(x|\theta)$, then evaluates the posterior density based on it. Sequential neural ratio estimation (SNRE) method \citep{hermans2020likelihood, durkan2020contrastive, miller2022contrastive} computes the approximation ratio of likelihood without evaluating the true likelihood. The neural posterior estimation (NPE) method approximates the true posterior density $p(\theta|x)$ by a neural posterior estimator $q_{F(x,\phi)}(\theta)$, where $F(x,\phi)$ denotes a neural network structure. When the inference problem is performed only on an observation $x_o$, data efficiency can be improved with sequential training schemes proposed in the sequential neural posterior estimation (SNPE) method \citep{papamakarios2016fast, lueckmann2017flexible, apt}: model parameters are drawn from the proposal, a more informative distribution about $x_o$ compared to the prior distribution. Meanwhile, SNPE requires a modification of the loss function called the `correction step' compared to NPE to ensure that $q_{F(x,\phi)}(\theta)$ is the approximation of the target posterior $p(\theta|x)$. Many approaches have been developed to solve this problem. These methods are summarized and reviewed in \citet{cranmer2020frontier} and they have been benchmarked in \citet{lueckmann2021benchmarking}.

Our work is based on the SNPE-B method proposed by \citet{lueckmann2017flexible}. It is worth noticing that SNPE-B proposed to use a calibration kernel to solve data inefficiency by assigning proper weights to each pair of samples $(\theta, x)$, leading to a better convergence of the neural network with the same amount of data. However, the use of calibration kernels may increase the variances of both the empirical loss and its gradient, making the training inefficient. Therefore, we propose to use an adaptive calibration kernel together with some variance reduction techniques. Specifically, we use Gaussian calibration kernels \citep{beaumont2002approximate} by selecting the calibration kernel rate adaptively. A theoretical analysis of the variance of the Monte Carlo estimator is provided to guide the design of variance reduction techniques, including defensive importance sampling and sample recycling.
Finally, a comprehensive comparative analysis of our proposed methods with some existing methods is carried out on numerical experiments. The experimental results show that the improvement of using either calibration kernels or variance reduction techniques is minor, suggesting the necessity of their combination. The code is available on GitHub\footnote{\url{https://github.com/Yifei-Xiong/Efficient-SNPE}}.

The rest of this paper is organized as follows. In Section \ref{section:background}, we review the basic SNPE formulation together with the SNPE-B algorithm and illuminate their limitations. In Section \ref{section:meth_ack_varred}, we firstly establish a theory to analyze the variance of calibration kernel technique, with which we identify three issues that could lead to an excessive variance of our method, including unstable density ratio, the trade-off of the bandwidth of calibration kernel and limited training samples. For these three problems, we propose corresponding solutions for each of them, including the adaptive calibration kernel method, sample recycling strategy, and defensive sampling. We also examine the mass leakage problem in this section and propose the corresponding solution.
In Section \ref{section:experiments}, some numerical experiments are conducted to support the advantages of our proposed methods, including a high-dimensional model with real-world dataset. Section \ref{section:discuss} concludes the main features of our strategies and points out some research directions in the future work.

\section{Background} \label{section:background}
Let $p(\theta)$ be the prior of the model parameter of interest, given an observed sample $x_o$ of interest. We aim to make an inference on the posterior distribution $p(\theta|x_o)\propto p(\theta) p(x_o|\theta)$. However, in most cases, the likelihood function $p(x|\theta)$ either does not have an explicit expression or is difficult to evaluate but can be expressed in the form of a `simulator'. In other words, given a fixed model parameter $\theta$, we can generate samples $x\sim p(x|\theta)$ without evaluating the likelihood. 

Our goal is to fit the posterior $p(\theta|x_o)$ by a tractable density. The KL divergence is commonly used to measure the discrepancy between two general densities $P(x), Q(x)$, which is defined as $$\mathcal{D}_{\mathrm{KL}}\left(P(x)\|Q(x)\right)  = \int P(x)\log \frac{ P(x)}{Q(x)}\mathrm{d}x.$$
If the likelihood $p(x|\theta)$ is tractable, one may find a good density estimation $q(\theta)$ in a certain family of distributions by minimizing the KL divergence $\mathcal{D}_{\mathrm{KL}}\left(p(\theta|x_o)\|q(\theta)\right)$ or $\mathcal{D}_{\mathrm{KL}}\left(q(\theta)\|p(\theta|x_o)\right)$. However, things become difficult when the likelihood is intractable. From another point of view, likelihood-free inference can be viewed as a problem of conditional density estimation. In such a framework, a conditional density estimator $q_{F(x, \phi)}(\theta)$, such as a neural network-based density estimator, where $\phi$ stands for the parameter of the density estimator, is used to approximate $p(\theta|x)$ \cite{papamakarios2017masked, dinh2016density, apt}. Now our objective is to minimize the \textit{average KL divergence} under the marginal $p(x)=\int p(\theta) p(x|\theta) \mathrm{d} \theta$, i.e.,
\begin{equation}\label{eq:npe}
\begin{aligned}
\mathbb{E}_{ p(x)}&\left[\mathcal{D}_{\mathrm{KL}}\left(p(\theta|x)\| q_{F(x,\phi)}(\theta)\right)\right]\\
&= \iint p(x)p(\theta|x) \left(\log  p(\theta|x)-\log q_{F(x,\phi)}(\theta)\right) \mathrm{d}x\mathrm{d}\theta\\
&=-\mathbb{E}_{p(\theta, x)}\left[\log q_{F(x,\phi)}(\theta)\right] + \iint  p(\theta,x) \log p(\theta|x) \mathrm{d}x\mathrm{d}\theta\\
& := L(\phi) + \iint p(\theta,x) \log p(\theta|x) \mathrm{d}x\mathrm{d}\theta,
\end{aligned}
\end{equation}
where the term
$$L(\phi)=-\mathbb{E}_{p(\theta, x)}\left[\log q_{F(x,\phi)}(\theta)\right]$$
is defined as the loss function.
Since the second term in Eq.\eqref{eq:npe} is unrelated to $\phi$, minimizing the loss function $L(\phi)$ is equivalent to minimizing the expected KL divergence of $p(\theta|x)$ and $q_{F(x, \phi)}(\theta)$ concerning $x\sim p(x)$. Since $L(\phi)$ is intractable, one common way is to use the empirical loss function obtained by Monte Carlo simulation, i.e.,
\begin{equation*}
\hat{L}(\phi) = -\frac{1}{N}\sum_{i=1}^N \log q_{F(x_i,\phi)} (\theta_i),
\end{equation*}
where the training data $\{(\theta_i, x_i)\}_{i=1}^N$ are sampled from the joint probability density $p(\theta, x)=p(\theta)p(x|\theta)$. After training, $p(\theta|x_o)$ can be easily approximated by $q_{F(x_o, \phi)}(\theta)$. This strategy differs from the way of fitting the posterior $p(\theta|x_o)$ directly.

Since the conditional density estimation at $x_o$ is ultimately used, a distribution $\tilde{p}(\theta)$ which is more informative about $x_o$ instead of $p(\theta)$ is preferred to generate $\theta$, which is called the proposal. After initializing $\tilde{p}(\theta)$ as $p(\theta)$, we then want the approximation of $p(\theta|x_o)$ to serve as a good proposal in the subsequent round. This conditional density estimation with an adaptively chosen proposal is called the SNPE. However, by replacing the joint density $p(\theta,x)$ with $\tilde{p}(\theta,x) = \tilde{p}(\theta)p(x|\theta)$ in $L(\phi)$, it turns out that $q_{F(x_o,\phi)}(\theta)$ approximates the so-called \textit{proposal posterior}, which is defined as
\begin{equation}
\tilde{p}(\theta|x)=p(\theta|x)\frac{\tilde{p}(\theta)p(x)}{p(\theta)\tilde{p}(x)},
\end{equation}
where $\tilde p(x)=\int \tilde p(\theta) p(x|\theta)\mathrm{d}\theta$. Hence adjustment towards the loss function $L(\phi)$ is required. The mainstream approaches to address this problem can be mainly categorized into three methods: SNPE-A \cite{papamakarios2016fast}, SNPE-B \cite{lueckmann2017flexible}, and automatic posterior transformation (APT, also known as SNPE-C) \cite{apt}. 
These methods iteratively refine the parameter $\phi$ and proposal $\tilde{p}(\theta)$ over a series of iterations, commonly referred to as the `rounds' of training. In $r$-th round, a distinct proposal $\tilde{p}_r(\theta)$ is used, leading to different loss functions.

The SNPE-A method, as outlined by \citet{papamakarios2016fast}, constrains the distribution $q_{F(x,\phi)}(\theta)$ to a specific structure, such as a mixture of Gaussians (MoG). Where the prior distribution $p(\theta)$ and proposal distribution $\tilde{p}(\theta)$ are restricted to be selected from the uniform distribution and Gaussian distribution. This approach enables a closed-form solution post-training, given by 
\begin{equation*}
q^\prime_{F(x, \phi)}(\theta) = q_{F(x, \phi)}(\theta) \frac{p(\theta) \tilde p(x)}{\tilde p(\theta) p(x)},
\end{equation*}
subsequently, an approximation of the posterior $p(\theta | x_o)$ is achieved through $q^\prime_{F(x_o, \phi)}(\theta)$. However, a limitation arises when the covariance matrix of some components in $q^\prime_{F(x_o, \phi)}(\theta)$ becomes non-positive definite. This occurs if the proposal $\tilde p(\theta)$ is narrower than some component of $q_{F(x_o, \phi)}(\theta)$, leading to a potential failure of the method.

\citet{apt} proposes the APT method, whose modification of neural posterior is formulated as 
\begin{equation}\label{eq:loss_apt}
    \tilde q_{F(x,\phi)}(\theta) = q_{F(x,\phi)}(\theta) \frac{\tilde p(\theta)}{p(\theta)} \frac{1}{Z(x,\phi)},
\end{equation}
where $Z(x,\phi)=\int q_{F(x,\phi)}(\theta) \tilde p(\theta)/p(\theta)\mathrm{d}\theta$ represents a normalization constant. Following \citep[][Proposition 1]{papamakarios2016fast}, when $\tilde q_{F(x,\phi)}(\theta)$ converges to the \textit{proposal posterior} $\tilde{p}(\theta|x)\propto\tilde p(\theta)p(x|\theta)$, it implies that $q_{F(x,\phi)}(\theta)$ converges to $p(\theta|x)$. To compute $Z(x,\phi)$, \cite{papamakarios2016fast} proposes the atomic method, which enables the analytical computation of $Z(x,\phi)$ by setting $\tilde{p}(\theta)$ as a discrete distribution. However, this method necessitates nested simulation when computing $Z(x, \phi)$, which implies extra simulation and computational costs.

The SNPE-B method, introduced by \citet{lueckmann2017flexible}, avoids this issue by incorporating an additional density ratio ${p(\theta)}/{\tilde p(\theta)}$ acting as importance weight and a calibration kernel $K_\tau(x,x_o)$ \citep{blum2010non} directing its attention towards simulated data points $x$ that exhibit proximity to $x_o$ into the loss function. It is formulated with
\begin{align}
\tilde{L}(\phi)&:=\mathbb{E}_{\tilde p(\theta, x)}\left[-K_\tau(x, x_o)\frac{p(\theta)}{\tilde p(\theta)}\log q_{F(x,\phi)}(\theta)\right] \notag \\
&=\int  p(x)K_\tau(x, x_o)\left[\int p(\theta|x)\log\left(\frac{ p(\theta|x)}{q_{F(x,\phi)}(\theta)}\right)\mathrm{d}\theta\right]\mathrm{d}x\notag + \mathrm{Const.} \notag\\
&=C_\tau\mathbb{E}_{ p_{\tau, x_o}(x)}\left[\mathcal{D}_{\mathrm{KL}}\left(p(\theta|x)\| q_{F(x,\phi)}(\theta)\right)\right]+ \mathrm{Const}, \label{eq:calib}
\end{align}
where $p_{\tau, x_o}(x):={p(x)K_\tau(x, x_o)}/{C_\tau}$ and $C_\tau:={\int p(x)K_\tau(x, x_o) \mathrm{d} x}$. It is worth noting that the density $p_{\tau, x_o}(x)$ will assign more mass in the neighborhood of $x_o$. The application of the calibration kernel enhances data efficiency by concentrating the KL divergence around $x_o$, thereby improving the accuracy of the density estimator. Subsequently, its unbiased estimator is then given as follows
\begin{equation}
\hat{\tilde{L}}(\phi)=
-\frac{1}{N}\sum_{i=1}^N K_\tau(x_i, x_o)\frac{p(\theta_i)}{\tilde p(\theta_i)}\log q_{F(x_i,\phi)}(\theta_i). \label{eq:calib_estimator}
\end{equation}
Consequently, the model parameter $\theta$ can be drawn from any proposal \cite{papamakarios2019sequential}, and this proposal may even be chosen using active learning \cite{gutmann2016bayesian, lueckmann2019likelihood, jarvenpaa2019efficient}. The tractability of the originally proposed calibration kernel depends on the structure of $q_{F(x,\phi)}(\theta)$. Unfortunately, it is intractable for most network structures. Consequently, aiming to compare this method with ours in this paper with flexible network structures, we use its calibration-kernel-removed version as the standard SNPE-B.

However, the incorporation of an additional density ratio in Eq.\eqref{eq:calib} usually results in unstable training. This instability is evidenced both in our numerical experiments presented in Figure \ref{fig:valid_nsf} and as described in the numerical experiments of \citet{apt}. One main reason is due to the small density of $\tilde{p}(\theta)$ in the denominator. To address these issues, we propose a suite of methodologies. This includes the implementation of defensive sampling to reduce variance and the adoption of a sample recycling strategy based on multiple-importance sampling. 

\section{Methods} \label{section:meth_ack_varred}

\subsection{Effects of calibration kernels} \label{calib}
The Eq.\eqref{eq:calib} indicates the expected approximation of $p(\theta|x)$ by $q_{F(x,\phi)}(\theta)$. However, since only the posterior $p(\theta|x)$ at the specific observed sample $x_o$ is of interest, it is desirable to minimize the KL divergence around $x_o$ rather than over the entire distribution $p(x)$. One approach to accomplish this is to incorporate a kernel function $K_\tau(x, x_o)$ which was introduced by \citet{blum2010non}, quantifying the proximity of the data point $x$ to $x_o$, the point of interest. \citet{lueckmann2017flexible} constructed a kernel that leverages the influence of the distance between $x$ and $x_o$ on the KL divergence between the associated posteriors, represented as $\exp(-\mathcal{D}_{\mathrm{KL}}(q_{F(x,\phi^{(r)})}(\theta) || q_{F(x_o,\phi^{(r)})}(\theta))/\tau)$ in the $r$-th round, where $\tau$ is a positive parameter called the $\textit{calibration kernel rate}$. In kernel density estimation, it is more universally known as the $\textit{bandwidth}$. In this approach, two data sets are considered similar if the current estimation by the density network assigns similar posterior distributions to them. However, this method is analytically tractable only when $q$ is chosen in a special form, such as a Gaussian mixture. Limited flexibility in $q$ can lead to decreased performance in the final approximation \citep{lueckmann2021benchmarking, deistler2022truncated}. For more complex density forms such as flow-based density \cite{papamakarios2017masked, durkan2019neural}, the computation of KL divergence becomes intractable. A simple and popular choice is the Gaussian kernel \citep{beaumont2002approximate}, defined as $K_\tau(x, x_o)=(2\pi)^{-d/2}\tau^{-d}\exp(-\|x-x_o\|_2^2/(2\tau^2))$, where $d$ represents the data dimension of data $x$.

To account for the impact of variations in the magnitudes of different components of $x$ on the distance $\|x-x_o\|_2^2$, the Mahalanobis distance can be employed \cite{peters2012sequential, erhardt2016modelling}. It is defined as $\|x-x_o\|_\mathrm{M}^2=(x-x_o)^\top\Sigma^{-1}(x-x_o)$, where $\Sigma$ denotes the covariance matrix of $x$. The covariance matrix can be estimated by $\hat\Sigma=\frac{1}{N-1}\sum_{i=1}^N(x_i-\bar{x})(x_i-\bar{x})^\top$, where $\bar x=\frac{1}{N}\sum_{i=1}^N x_i$. The Mahalanobis distance is invariant to scaling compared with Euclidean distance. Using the Mahalanobis distance, the modified Gaussian kernel function is given by $K_\tau(x, x_o)=(2\pi)^{-d/2}|\Sigma|^{-1/2}\tau^{-d}\exp(-(x-x_o)^\top\Sigma^{-1}(x-x_o)/(2\tau^2))$.
Figure \ref{fig:ack_plot} illustrates the impact of the calibration kernel rate on the sample weight when selecting different $\tau$. This can be seen as a rare event sampling method, where rare events refer to a neighborhood of $x_o$ in the sense of Mahalanobis or Euclidean distance.

\begin{figure}[t]
\includegraphics[width=1.05\textwidth]{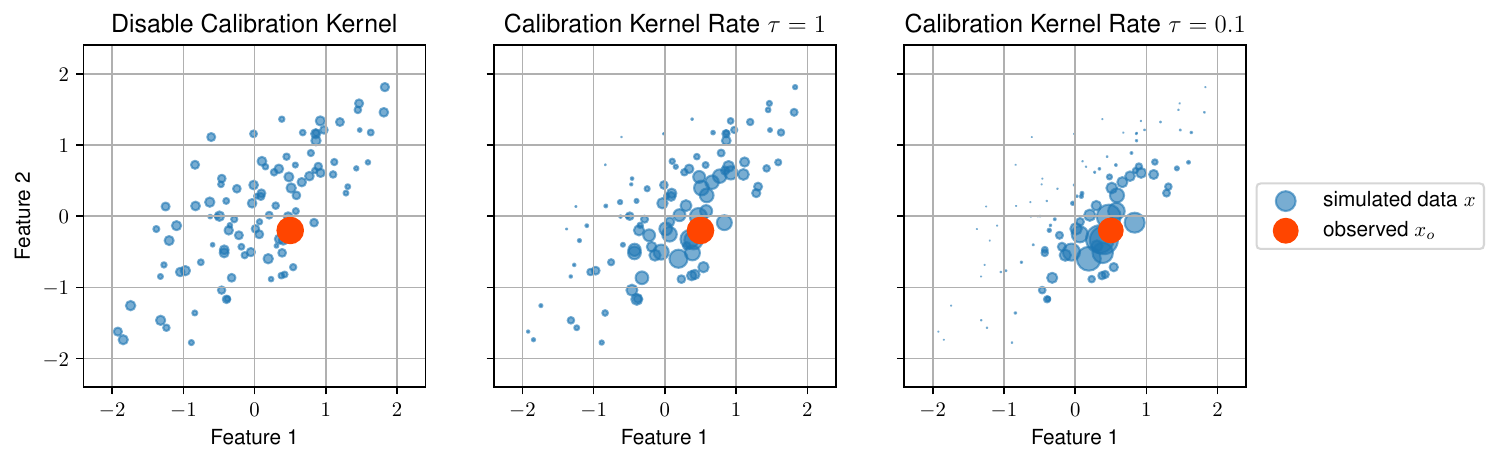}
\caption{The calibration kernel adjusts the sample weights around $x_o$, where larger points are assigned higher weights. Left plot: only the density ratio $p(\theta)/\tilde p(\theta)$ is used to weight each sample, without considering the calibration kernel. Middle plot: the calibration kernel with rate $\tau=1$ is incorporated to adjust the sample weights. Right plot: the calibration kernel is applied with rate $\tau=0.1$.}
\label{fig:ack_plot}
\end{figure}

The calibration kernel rate $\tau$ has an impact on the variance of the empirical loss and its gradient. In extreme situations where $\tau$ tends to 0, the encountered distribution degenerates to a Dirac function, and the only observation available in training is $x_o$, indicating that in most cases there will be few or even no samples in the training dataset. When a too large $\tau$ is used, the distribution becomes flat, and hence similar weights will be assigned to all samples making the inference inefficient. To examine the impact of calibration kernels on the variance of the empirical loss in greater detail, we have the following theorem.

\begin{thm} \label{thm:calib_analysis}
Let $g(\theta,x)$ be any function and 
\begin{align*}\mu(\tau)  & :=\mathbb{E}_{\tilde{p}(\theta, x)}\left[K_\tau\left(x, x_o\right) g(\theta, x)\right], \\
V(\tau) & :=\operatorname{Var}_{\tilde{p}(\theta, x)}\left[K_\tau\left(x, x_o\right) g(\theta, x)\right],\\
h_g(x)&:=\int g(\theta, x)\tilde p(\theta, x)\mrd\theta.
\end{align*}
Assume that $h_g(x)$ and $h_{g^2}(x)$ are continuously differentiable twice with bounded second-order derivatives.
As $\tau \to 0$, we have
\begin{align}
\mu(\tau) & =h_g\left(x_o\right)+\mathcal{O}\left(\tau^2\right), \label{eq:mean_tau} \\ 
V(\tau) & =C h_{g^2}\left(x_o\right) \tau^{-d}+\mathcal{O}\left(\tau^{2-d}\right), \label{eq:var_tau}
\end{align}
where $C=(2 \pi)^{-d / 2} 2^{-d / 2}|\Sigma|^{-1/2}$.
\end{thm}
The proof is detailed in Appendix~\ref{appendix:calib_analysis}. As a result, by plugging
\begin{align*}
    g(\theta,x) = \frac{p(\theta)}{\tilde{p}(\theta)} \log q_{F(x_o,\phi)}(\theta),
\end{align*}
into Eq.\eqref{eq:var_tau} divided by $N$, we obtain the variance of the estimator $\hat{\tilde{L}}(\phi_t)$ given by Eq.\eqref{eq:calib_estimator}, i.e.,
\begin{align}
\operatorname{Var}\left(\hat{\tilde{L}}\left(\phi_t\right)\right)&=V_1\left(\phi_t\right) \tau^{-d} N^{-1}+\mathcal{O}\left(\tau^{2-d} N^{-1}\right), \label{eq:var_of_loss_1}
\end{align}
where
\begin{align}\label{eq:var1}
V_1(\phi) 
& =C p\left(x_o\right) \int \frac{p(\theta)}{\tilde{p}(\theta)}\left( \log q_{F\left(x_o, \phi\right)}(\theta)\right)^2 p(\theta | x_o) \mrd \theta.
\end{align}
Similarly, the variance of the gradient estimator $\operatorname{Var}(\nabla_\phi \hat{\tilde{L}}(\phi_t))$ is given as
\begin{align}
\operatorname{Var}\left(\nabla_\phi \hat{\tilde{L}}\left(\phi_t\right)\right)&=V_2\left(\phi_t\right) \tau^{-d} N^{-1}+\mathcal{O}\left(\tau^{2-d} N^{-1}\right), \label{eq:var_of_loss_2}
\end{align}
with
\begin{equation}\label{eq:var2}
V_2(\phi) =C p\left(x_o\right) \int \frac{p(\theta)}{\tilde{p}(\theta)}\left(\nabla_\phi \log q_{F\left(x_o, \phi\right)}(\theta)\right)^2 p(\theta | x_o) \mrd \theta.
\end{equation}

From Eq.\eqref{eq:var_of_loss_1} and Eq.\eqref{eq:var_of_loss_2}, there are two issues to be solved to diminish the variance of SNPE-B. The first issue, as indicated by Eq.\eqref{eq:var1} and Eq.\eqref{eq:var2}, is that small-density areas of $\tilde{p}(\theta)$ can lead to excessive variance.
The second issue is attributed to the sample size $N$. Although increasing $N$ can be advantageous for variance reduction, the cost of conducting multiple simulations can be extremely high. The sampling recycling technique \cite{apt} has been demonstrated to be a useful approach to tackle this problem. Owing to the special formulation of importance sampling, further improvements can be achieved.

How to choose an appropriate $\tau$? Eq.\eqref{eq:mean_tau} indicates that a sufficiently small $\tau$ allows us to infer the true value of the posterior density at $x_o$. However, as $\tau \rightarrow 0$ in Eq.\eqref{eq:var_tau}, $V(\tau)=\mathcal{O}(\tau^{-d})$ increases exponentially, leading to the curse of dimensionality. This is a trade-off to choose a proper $\tau$ to balance the efficiency and variance. To handle the curse of dimensionality, we propose to use low-dimensional summary statistics instead of the full data (see Appendix ~\ref{appendix:summary_stat} for details). However, this introduces additional bias, as detailed in \cite{fearnhead2012constructing}. In this paper, we do not look into the effects of using summary statistics for SNPE methods. We refer to \citet{fearnhead2012constructing} for a semi-automatic method of constructing summary statistics in the context of ABC.

\subsection{Reducing variance in SNPE-B: our proposed method}\label{sec:var_red}

\subsubsection{Defensive sampling-based variance control} \label{defensive}
It is shown in the estimator in Eq.\eqref{eq:calib_estimator} that a relatively small value of the proposal $\tilde p(\theta)$ for $\theta_i\sim \tilde p(\theta)$ can lead an exceedingly large value of the density ratio $p(\theta_i)/\tilde p(\theta_i)$. This, in turn, can cause a substantial increase in the variance of the loss function as it is indicated in Eq.\eqref{eq:var_of_loss_1} and Eq.\eqref{eq:var_of_loss_2}. It would be sufficient to address this issue by ensuring the heavy tail of the proposal $\tilde{p}(\theta)$ compared to the prior $p(\theta)$. As an essential problem in the importance sampling community, defensive sampling, which is reported in \cite{hesterberg1995weighted, owen2000safe}, has been proven to be a useful solution. Pick $\alpha$ with $0< \alpha< 1$ and choose
\begin{align}\label{eq:defensive}
p_\alpha(\theta) =  (1-\alpha) q_{F(x_o,\phi^*)}(\theta) + \alpha p_{\mathrm{def}}(\theta),
\end{align}
as the new proposal distribution rather than $q_{F(x_o,\phi^*)}(\theta)$, where $\phi^*$ is the optimal neural network parameter obtained from the last round, and $p_{\mathrm{def}}(\theta)$ is called the defensive density, with the `ideal property' of being easily sampled and avoiding excessively low density. A possible choice of $p_{\mathrm{def}}(\theta)$ is uniform distribution. 
From Eq.\eqref{eq:var1}, the variance of the estimator of the loss function $\operatorname{Var}(\hat{L}(\phi_t))$ is well controlled by noticing 
\begin{align}
\int \frac{p(\theta)}{p_\alpha(\theta)}\left( \log q_{F\left(x_o, \phi\right)}(\theta)\right)^2 p(\theta | x_o) \mrd \theta 
&\le \frac{1}{\alpha}\int \frac{p(\theta)}{p_{\mathrm{def}}(\theta)}\left( \log q_{F\left(x_o, \phi\right)}(\theta)\right)^2 p(\theta | x_o) \mrd \theta \notag \\
&\le \frac{M}{\alpha}\int \left( \log q_{F\left(x_o, \phi\right)}(\theta)\right)^2 p(\theta | x_o) \mrd \theta, \label{eq:ds_bd_var}
\end{align}
where the ratio $p(\theta)/p_{\mathrm{def}}(\theta)$ is assumed to be bounded by a constant $M>0$. The variance of the gradient estimator Eq.\eqref{eq:var2} can be controlled similarly. This method is summarized as Algorithm \ref{alg:defensive}.
\begin{algorithm}
\caption{Defensive sampling-based SNPE-B method}\label{alg:defensive}
\begin{algorithmic}[1]
\State Initialization: $\tilde p(\theta) := p(\theta)$, given $\alpha\in(0, 1)$, number of rounds $R$, simulations per round $N$
\For{$r=1,2,\cdots,R$}
\State sample $\{\theta_i\}_{i=1}^N$ from $\tilde p(\theta)$
\State sample $x_i\sim p(x|\theta_i),i=1,2,\cdots,N$, resulting in $\{x_i\}_{i=1}^N$
\State update $\phi^* := \mathop{\arg\min}_{\phi} -\frac{1}{N}\sum_{i=1}^N\frac{p(\theta_i)}{\tilde p(\theta_i)} K_\tau (x_i, x_o) \log q_{F(x_i, \phi)}(\theta_i)$
\State set $\tilde p(\theta)=(1-\alpha) q_{F(x_o, \phi^*)}(\theta) + \alpha p_{\mathrm{def}}(\theta)$ 
\EndFor \\
\Return $q_{F(x_o, \phi^*)}(\theta)$
\end{algorithmic}
\end{algorithm}

\subsubsection{Sample recycling and multiple importance sampling}\label{sec:MISR}

\sloppy One effective way to enhance training and sample efficiency is to train $q_{F(x,\phi)}(\theta)$ on samples from previous rounds. Specifically, let $N_r$ denote the number of parameter samples generated in the $r$th round, where the corresponding proposal is denoted as $\tilde p_r(\theta)$. The parameters sampled from the proposal in the $r$th round are denoted as $\theta_1^{(r)},\cdots,\theta_{N_r}^{(r)}$, along with the corresponding observations generated with the simulator $x_1^{(r)},\cdots,x_{N_r}^{(r)}$, where $r=1,\cdots,R$. Note that for any choice of proposal in rounds $k\leq r$, 
\begin{equation*}
-\frac{1}{N_k} \sum_{i=1}^{N_k} \frac{p(\theta_{i}^{(k)})}{\tilde{p}_k(\theta_i^{(k)})}K_\tau (x_i^{(k)},x_o)\log q_{F(x_i^{(k)},\phi)}(\theta_i^{(k)})
\end{equation*}
is an unbiased estimator of $\tilde{L}(\phi)$ given by Eq.\eqref{eq:calib}. Therefore, aggregating the loss functions in the first $r$ rounds may yield a better unbiased estimator of $\tilde{L}(\phi)$, given by
\begin{align}\label{eq:reuse1}
  L^{(r)}(\phi) 
  &=\frac{1}{r}\sum_{k=1}^r\left(-\frac{1}{N_k} \sum_{i=1}^{N_k} \frac{p(\theta_i^{(k)})}{\tilde p_k(\theta_i^{(k)})} K_\tau(x_i^{(k)}, x_o) \log q_{F(x_i^{(k)},\phi)} (\theta_i^{(k)})\right).
\end{align}

On the other hand, inspired by the multiple importance sampling method \cite{veach1995optimally}, we can construct a better estimator of the corresponding form
\begin{align}\label{eq:reuse_partition}
L_\omega^{(r)}(\phi) &=\sum_{k=1}^r\left(-\frac{1}{N_k} \sum_{i=1}^{N_k} \omega_k(\theta_i^{(k)})\frac{p(\theta_i^{(k)})}{\tilde p_k(\theta_i^{(k)})} K_\tau(x_i^{(k)}, x_o) \log q_{F(x_i^{(k)},\phi)} (\theta_i^{(k)})\right).
\end{align}
Here $\omega_k(\cdot)$ can be considered as a non-negative weight function and $\{\omega_k(\cdot)\}_{k=1}^r$ is a partition of unity, satisfying $\sum_{k=1}^r \omega_k(\theta)=1$ for all $\theta$. If all $\omega_k(\theta)=1/r$, $L_\omega^{(r)}(\phi)=L^{(r)}(\phi)$. One way of selecting weights with nearly optimal variance is to take $\omega_k(\theta)\propto N_k \tilde p_k(\theta)$, which is also known as the balance heuristic strategy \cite{veach1995optimally}. The balance heuristic strategy is highly effective as confirmed by the following theorem. We denote this unbiased estimator of the loss function based on the balance heuristic strategy as $L_{\mathrm{BH}}$, which has the following form
\begin{align}\label{eq:reusebh}
  L_{{\mathrm{BH}}}^{(r)}(\phi) 
  &=\sum_{k=1}^r\left(-\frac{1}{N_k} \sum_{i=1}^{N_k} \omega_k^{\mathrm{BH}}(\theta_i^{(k)})\frac{p(\theta_i^{(k)})}{\tilde p_k(\theta_i^{(k)})} K_\tau(x_i^{(k)}, x_o) \log q_{F(x_i^{(k)},\phi)} (\theta_i^{(k)})\right), 
\end{align}
where $$\omega_k^{\mathrm{BH}}(\theta)=\frac{N_k \tilde p_k(\theta)}{\sum_{k=1}^r N_k \tilde p_k(\theta)}.$$ This method is referred to as multiple importance sampling and recycling (MISR).
\begin{thm}
Let $L_\omega^{(r)}(\phi)$ and $L_{{\mathrm{BH}}}^{(r)}(\phi)$ be the loss estimators given in Eq.\eqref{eq:reuse_partition} and Eq.\eqref{eq:reusebh}, respectively. Then
\begin{equation*}\mathrm{Var}\left(L_{{\mathrm{BH}}}^{(r)}(\phi)\right) \le \mathrm{Var}\left(L_\omega^{(r)}(\phi) \right)+ \left(\frac{1}{\min_k N_k}-\frac{1}{\sum_{k=1}^r N_k}\right)\tilde{L}(\phi)^2,\end{equation*}
where $\tilde{L}(\phi)$ is given by Eq.\eqref{eq:calib}.
\end{thm}
\begin{proof}
This is an application of \citep[][Theorem 1]{veach1995optimally}. 
\end{proof}

\subsubsection{Adaptive calibration kernel based on an ESS criterion} \label{ess}

This section presents an adaptive selection method for determining an appropriate $\tau$ value in the training process. Recall that in Eq.\eqref{eq:mean_tau}, as $\tau\rightarrow 0$,
\begin{equation*}
L(\phi)=p(x_o)\mathcal{D}_{\mathrm{KL}}\left(p(\theta|x_o)\|q_{F(x_o,\phi)}(\theta)\right)+\mathcal{O}(\tau^2)+\mathrm{Const}.
\end{equation*}
As a result, when $\tau$ is small enough,
\begin{equation*}
\underset{\phi}{\mathrm{argmin}}\,L(\phi) \approx \underset{\phi}{\mathrm{argmin}} \,\mathcal{D}_{\mathrm{KL}}\left(p(\theta|x_o)\|q_{F(x_o,\phi)}(\theta)\right),
\end{equation*}
indicating that we should choose $\tau$ as small as possible. However, small $\tau$ will lead to a large variance of estimator according to Eq.\eqref{eq:var_of_loss_1}. 

In our basic setting, for a mini-batch of samples, the weight and the size of these samples serve as two indicators of their importance for training. The insight is that the greater the weight and sample size of a mini-batch of samples, the higher their importance in training. When selecting $\tau$, two types of samples can be obtained. If $\tau$ is large, the obtained samples will be a batch of samples with similar weights. If $\tau$ is small, the obtained samples will consist of a small portion of samples with larger weights and a large portion of samples with smaller weights. It may be difficult to compare the effects of these two types of samples since both indicators vary.

Effective sampling size (ESS) is proposed in \citet{kong1992note} and then popularized by \citet{liu1996metropolized}. By incorporating the sample weight and the sample size, it indicates `the real size of samples in training', which is defined as
\begin{equation*}
\mathrm{ESS}=\frac{\left(\sum_{i=1}^N w_i\right)^2}{\sum_{i=1}^N w_i^2},
\end{equation*}
where $w_i=({p(\theta_i)}/{\tilde p(\theta_i)})K_\tau(x_i, x_o)$ in Algorithm \ref{alg:defensive}. 

In the absence of MISR, the number of samples used in each round remains constant. A common criterion for the selection of ESS in this scenario is to define it as a fraction of the current iteration sample size $N$, that is, $\mathrm{ESS} = \gamma N$ for $\gamma\in(0,1)$.  
In each round, after simulating the dataset $\{(\theta_i, x_i)\}_{i=1}^N$, the value of $\tau$ is numerically determined by solving the equation $\mathrm{ESS}=\gamma N$ using the bisection method \citep{del2012adaptive}. This approach achieves a balance between the bias and variance of the estimator.

The introduction of MISR significantly alters the sampling framework. Here, the count of samples per round equates to the cumulative number of generated samples, denoted as $rN$. Instead of letting ESS grow with $r$ linearly, we find that a slower growth rate as $\log(r)$ empirically performs better. Therefore, we propose to set $\mathrm{ESS} = f(r) \gamma N$ in the case of MISR, where $f(r)=\log(r)+1$, so that $\mathrm{ESS}=\gamma N$ for the first round training. 

\subsection{Solving mass leakage: Parameter space transformation}\label{sec:pst}
In some cases, the support of the prior of the model parameter, $\mathrm{supp}(p(\theta))=\{\theta\in\mathbb{R}^n|p(\theta)>0\}$, is a bounded area, where $n=\mathrm{dim}(\theta)$, but $q_{F(x,\phi)}(\theta)$ usually represents a density whose support is unbounded, as does $\tilde p(\theta)$. When a sample pair $(\theta,x)$ falls outside $\mathrm{supp}(p(\theta))$, its weight will be assigned as $0$, effectively rejecting the sample. However, this can result in wasted samples and data inefficiency, particularly when running the simulator is expensive. Moreover, the mismatched support sets can also lead to mass leakage. 

One feasible solution to this problem is to use a truncated density and estimate the normalization factor after training $q_{F(x,\phi)}(\theta)$ (see \cite{apt}). However, this method can also lead to a significant amount of mass being leaked outside the prior support. Another efficient method is to sample only from the truncated proposal \cite{deistler2022truncated}, which requires an additional computational step to reject samples and may still result in mass leakage. These approaches may generate samples from the proposal that lie outside $\mathrm{supp}(p(\theta))$ and require additional post-processing steps.

One way to essentially solve mass leakage is to remap the parameter space. Suppose there exists a reversible transformation $h(\cdot)$ such that the density function of $\hat\theta := h(\theta)$ is $\hat p(\hat\theta))$ and satisfies $\mathrm{supp}(\hat p(\hat\theta))=\mathbb{R}^n$. If $h(\cdot)$ is strictly monotone and invertible on each dimension of $\mathrm{supp}(p(\theta))$, then $\hat\theta$ has a density
\begin{equation*}
\hat p(\hat\theta)=p(h^{-1}(\hat\theta))\left\lvert\frac{\partial}{\partial \hat\theta}h^{-1}(\hat\theta)\right\rvert.
\end{equation*}
Using this transformation, it is possible to train $q_{F(x,\phi)}(\theta)$ on an unbounded support space derived from the transformed input. Additionally, once the posterior density $\hat p(\hat\theta|x_o)$ has been obtained in the transformed space, we can transform it back to the original parameter space using:
\begin{equation*}
p(\theta|x_o)=\hat p\left(h(\theta)|x_o\right)\left\lvert\frac{\partial}{\partial \theta}h(\theta)\right\rvert.
\end{equation*}
This trick requires an invertible monotonic mapping $h(\cdot)$. In a simple case of $$\mathrm{supp}(p(\theta))=[a_1, b_1]\times\cdots\times[a_n, b_n]\subset \mathbb{R}^n,$$ we can take $h(\theta)=\left(\ln(\frac{\theta_1-a_1}{b_1-\theta_1}),\cdots,\ln(\frac{\theta_n-a_n}{b_n-\theta_n})\right)^\top$ to transform $\mathrm{supp}(p(\theta))$ into $\mathbb{R}^n$. In complex cases of $\mathrm{supp}(p(\theta))$, it may be difficult to construct an invertible monotonic mapping $h(\cdot)$. To get rid of this, $h$ can be constructed on a rough region $K\subset \mathbb{R}^n$ such that $\mathrm{supp}(p(\theta))\subset K$, and one continues to combine with the truncation method \cite{apt, deistler2022truncated}. This prevents the density from leaking in regions $\mathbb{R}^n\backslash K$. If the difference between $K$ and $\mathrm{supp}(p)$ is small, it can further prevent a large amount of mass from being leaked outside.

The transformation of parameter space is sometimes necessary since drawing samples from the posterior distribution that lies inside the support of the prior is time-consuming in certain circumstances as it is reported in \citet{glockler2022variational}. 

\section{Experiments} \label{section:experiments}

In this section, we conduct a series of comparative analyses of our proposed strategies with existing methods to demonstrate their effectiveness. Since our strategies are built upon the SNPE-B method, hence we begin with a simulation study. Subsequently, we then compare our methods with other likelihood-free Bayesian computation methods such as SNPE-A \cite{papamakarios2016fast}, APT \cite{apt} and SNL \cite{papamakarios2019sequential}.

We employ neural spline flows (NSFs) \cite{durkan2019neural} as the conditional density estimators, which is highly flexible and has been widely applied in related research \cite{lueckmann2021benchmarking, deistler2022truncated}. Our NSFs follow the structure from the \texttt{SBIBM} library\footnote{\url{https://github.com/sbi-benchmark/sbibm}}, which consists of 5 layers. Each layer is constructed using two residual blocks with 50 units and \texttt{ReLU} activation function. With 10 bins in each monotonic piece-wise rational-quadratic transform, the tail bound is set to 20. Notably, our strategies have no restrictions on the type of $q_{F(x, \phi)}(\theta)$. Other conditional neural densities, such as masked autoregressive flows \citep{papamakarios2017masked} and mixtures of Gaussians \citep{bishop1994mixture, papamakarios2016fast} can also work.

We evaluate the performance of our proposed strategies on five likelihood-free Bayesian models: M/G/1 model \cite{shestopaloff2014bayesian}, Lotka-Volterra model \cite{lotka1920analytical}, SLCP model\cite{papamakarios2019sequential}, Gaussian linear model \citep{lueckmann2021benchmarking}, and state-space model \citep{rodrigues2020likelihood}. Where the previous four models come from the standard benchmark provided by the \texttt{SBIBM} library. Moreover, to test our method in a more complex and realistic case, we examine a state-space model designed for the inference of rental prices with intractable g-and-k distributions \citep{rodrigues2020likelihood}.  
We present numerical results for the M/G/1 and Lotka-Volterra models using synthetic data in Sections \ref{sec:exp_abla} and \ref{sec:exp_compare}, and provide results for the state-space model applied to a real-world rental price dataset in Section \ref{sec:exp_realistic}.
Numerical experiments on the SLCP and Gaussian linear models and additional experimental results are deferred to Appendices~\ref{appendix:model_detail} and \ref{appendix:exp_detail}, respectively.

\textbf{M/G/1 model.} The model describes a single server's processing of a queue of continuously arriving jobs. Define $I$ as the total number of jobs that need to be processed, and denote by $s_i$ the processing time required for job $i$. Let $v_i$ be the job's arrival time in the queue, and $d_i$ be the job's departure time from the queue. They satisfy the following conditions
\begin{align*}
s_i & \sim \mathcal{U} (\theta_1, \theta_1 + \theta_2), \\
v_i - v_{i-1}  & \sim \mathrm{Exp}(\theta_3), \\
d_i - d_{i-1} & = s_i + \max(0, v_i - d_{i-1}).
\end{align*}
In our experiments, we configure $I = 50$ and select the summary statistics $S(x)$ to be the logarithms of 0th, 25th, 50th, 75th and 100th percentiles of the set of inter-departure times. The prior distribution of the parameters is
\begin{align*}
\theta_1 \sim \mathcal{U} (0, 10),\
\theta_2 \sim \mathcal{U} (0, 10),\
\theta_3 \sim \mathcal{U} (0, 1/3),
\end{align*}
and our experiments choose ground truth parameter as $\theta^* = (1,\ 4,\ 0.2)$. The observed summary statistic $S(x_o)$ simulated from the model with ground truth parameter $\theta^*$ is
\begin{equation*}
S(x_o) = (0.0929,\ 0.8333,\ 1.4484,\ 1.9773,\ 3.1510),
\end{equation*}
and the corresponding standard deviation underground truth parameter $\theta^*$ (based on 10,000 simulations) is
\begin{equation*}
s = (0.1049,\ 0.1336,\ 0.1006,\ 0.1893,\ 0.2918).
\end{equation*}

\textbf{Lotka-Volterra model.} 
This model represents a Markov jump process that characterizes the dynamics of a predator population interacting with a prey population, with four parameters denoted as $\theta=(\theta_1,\cdots,\theta_4)$. Let $X$ be the number of predators, and $Y$ be the number of prey. The model posits that the following events can occur:

\begin{itemize}
\item the birth of a predator at a rate $\exp(\theta_1)XY$, resulting in an increase of $X$ by one.
\item the death of a predator at a rate proportional to $\exp(\theta_2)X$, leading to a decrease of $X$ by one.
\item the birth of a prey at a rate proportional to $\exp(\theta_3)Y$, resulting in an increase of $Y$ by one.
\item the consumption of a prey by a predator at a rate proportional to $\exp(\theta_4)XY$, leading to a decrease of $Y$ by one.
\end{itemize}
Following the experimental details in \citet{papamakarios2019sequential}, we initialize the predator and prey populations as $X = 50$ and $Y = 100$ respectively. We perform simulations of the Lotka-Volterra model using the Gillespie algorithm \cite{gillespie1977exact} over a duration of 30 time units, and record the populations at intervals of 0.2 time units, resulting in two time series of 151 values each. The resulting summary statistics $S(x)$ is a 9-dimensional vector that comprises the following time series features: the log mean of each time series; the log variance of each time series; the autocorrelation coefficient of each time series at lags 0.2 and 0.4 time units; the cross-correlation coefficient between the two time series. 

In our experiments, the prior distribution of the parameters is set to $\mathcal{U}(-5, 2)^4$, the ground truth parameters are
\begin{equation*}
\theta^* = (\log 0.01,\  \log 0.5, \  \log 1, \  \log 0.01),
\end{equation*}
the observed summary statistics $S(x_o)$ simulated from the model with ground truth parameters $\theta^*$ are
\begin{equation*}
S(x_o) = (4.6431,\ 4.0170,\ 7.1992,\ 6.6024,\ 0.9765,\ 0.9237,\ 0.9712,\ 0.9078,\ 0.0476),
\end{equation*}
and the corresponding standard deviation under ground truth parameters $\theta^*$ (based on 10,000 simulations) are
\begin{equation*}
s = (0.3294,\ 0.5483,\ 0.6285,\ 0.9639,\ 0.0091,\ 0.0222,\ 0.0107,\ 0.0224,\ 0.1823).
\end{equation*}

To evaluate the degree of similarity between the approximate posterior distribution $q_{F(x_o, \phi)}(\theta)$ and the true posterior distribution $p(\theta|x_o)$ given the observed data, we utilize maximum mean discrepancy (MMD) \cite{gretton2012kernel} and classifier two sample tests (C2ST) \citep{lopez2016revisiting} as a discriminant criterion. As reported in \citet{bischoff2024practical}, both metrics measure the posterior approximation errors, but C2ST is sensitive to subtle differences, potentially revealing nuances that MMD may miss due to kernel dependence. Another criterion for measurement is the log median distance (LMD) between $x_o$ and $x$ drawn from $p(x|\theta)$, where $\theta$ is sampled from $q_{F(x_o, \phi)}(\theta)$. We normalize the calculation of LMD using the standard deviation of each component of the samples due to the profound difference scales across dimensions and models. In scenarios where the true posterior distribution $p(\theta|x_o)$ is intractable even with knowledge of the sample generation process, we employ the negative log probability (NLOG) of the true parameters $\theta^*$. In this case, the observation $x_o$ is sampled from $p(x|\theta^*)$. All four metrics are the lower the better. We provide a caveat that these metrics suffer from various pitfalls and refer readers to \cite{lueckmann2021benchmarking,hermans2021trust, bischoff2024practical} for a detailed discussion. We also report the result by plotting the posterior marginal in Appendix \ref{appendix:exp_detail}.

Our numerical experiments were conducted on a computer equipped with a single GeForce RTX 2080s GPU and an i9-9900K CPU. The training and inference processes of the model were primarily implemented with \href{https://pytorch.org/}{\texttt{Pytorch}}.

In the training process, we find that a batch size of 1000, a sample size of $N = 1000$, and a total of $R = 20$ round generally yield the best results across all methods. Therefore, we fix these values for all experiments. We set $\alpha=0.2$ in the defensive density \eqref{eq:defensive} and  $\gamma=0.5$ in the ESS. In each round, we randomly pick 5\% of the dataset as the validation set. We follow the early stop criterion proposed by \citet{papamakarios2019sequential}, which terminates the training if the loss value on the validation data does not decrease after 20 epochs in a single round. For the stochastic gradient descent optimizer, we use \texttt{Adam} \cite{kingma2014adam} with a learning rate of $1 \times 10^{-4}$, and a weight decay of $1 \times 10^{-4}$. We report these metrics with error bars representing the mean with the upper and lower quarterlies after 50 times of repetition, where different methods were employed with the same seed.

\subsection{Simulation study on the SNPE-B method}\label{sec:exp_abla}

\begin{figure}[!t]
\begin{subfigure}[t]{1.00\textwidth}
  \vspace{0pt}
    \includegraphics[width=1.00\textwidth]{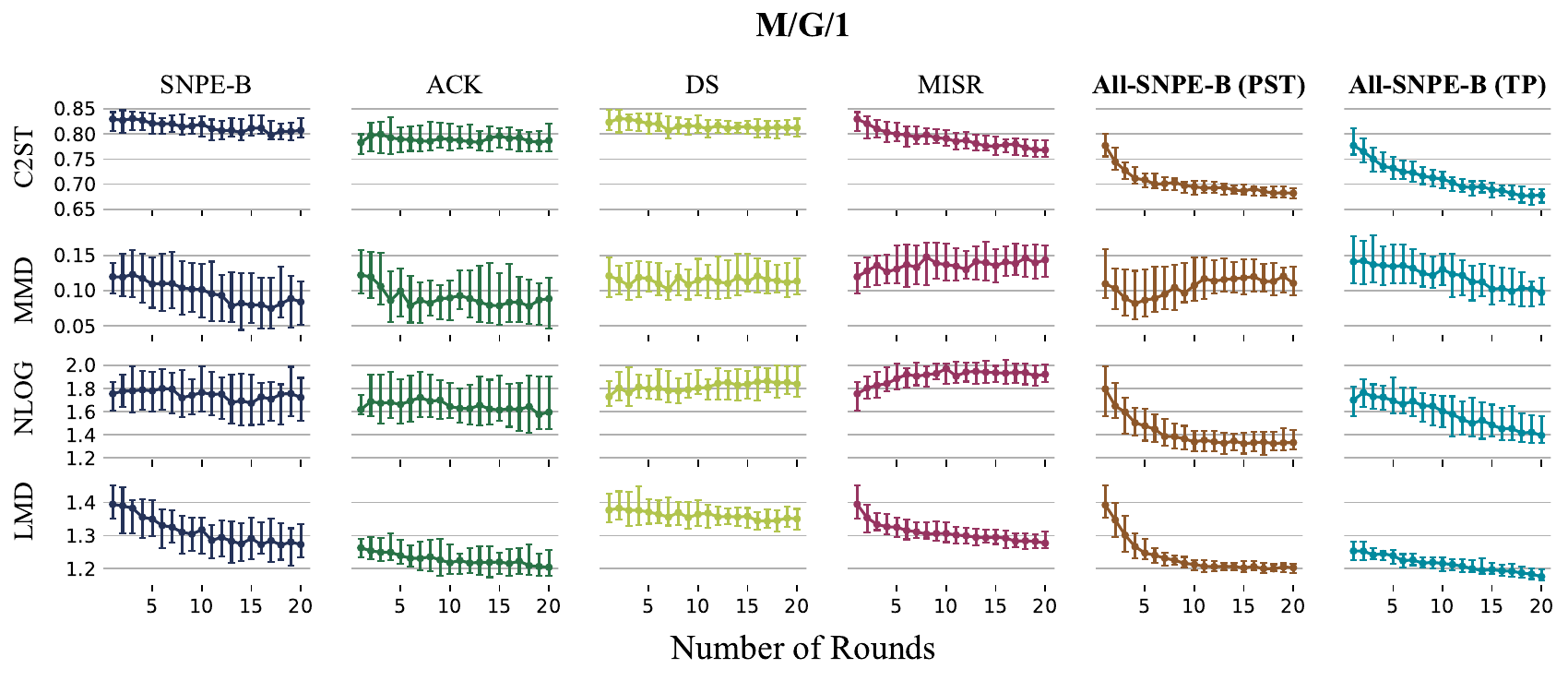}
\end{subfigure}
\hspace{-1.01\textwidth}
\begin{subfigure}[t]{0.02\textwidth}
  \vspace{5pt}
    \textbf{A}
\end{subfigure}

\begin{subfigure}[t]{1.00\textwidth}
  \vspace{0pt}
    \includegraphics[width=1.00\textwidth]{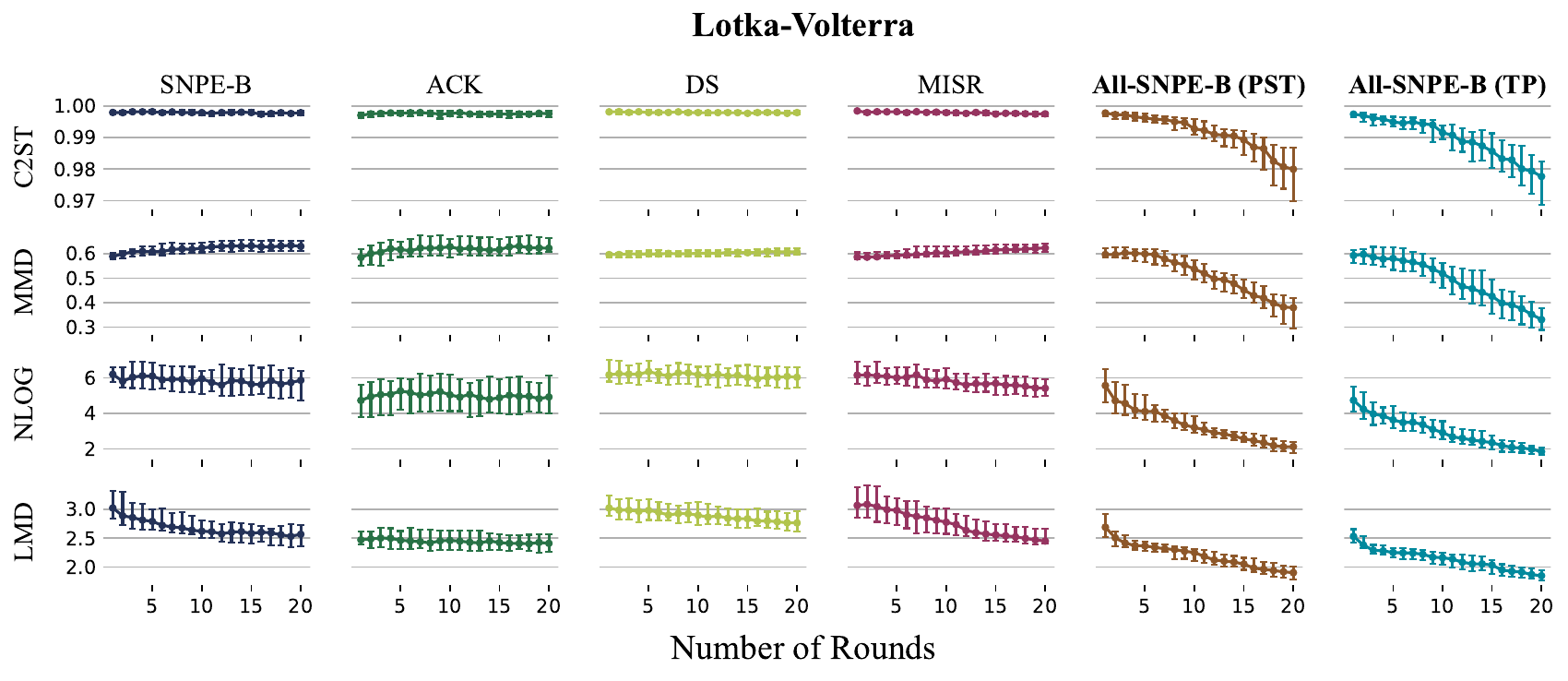}
\end{subfigure}
\hspace{-1.01\textwidth}
\begin{subfigure}[t]{0.02\textwidth}
  \vspace{5pt}
    \textbf{B}  
\end{subfigure}
\caption{\textbf{Simulation experiments on proposed strategies.} \textbf{A.} Performance on M/G/1 queuing model. \textbf{B.} Performance on Lotka-Volterra predation model. The horizontal axis represents the round of training and the error bars represent the mean with the upper and lower quarterlies. Our proposed method exhibits better performance compared to the original method.}
\label{fig:valid_nsf}
\end{figure}

In this section, we provide a detailed comparison of the performance of our proposed strategies, including adaptive calibration kernels (ACK), defensive sampling (DS), and multiple importance sampling and recycling (MISR) methods with the standard SNPE-B method. In some cases, the support of the ground truth posterior is bounded, while the neural posterior always has an unbounded support. To circumvent this, we consider two strategies to account for the inconsistency of support. The first one is known as truncated posterior (TP), in which we directly truncate the neural posterior after training, this method suffers from low sampling efficiency and leakage of mass. The second one is the parameter space transformation (PST), which is detailed in Section \ref{sec:pst}.
To study the difference between TP and PST, we additionally examine these two cases with all other strategies in our experiment, which are represented as `All-SNPE-B (TP)' and `All-SNPE-B (PST)', respectively. `SPNE-B' in the following figures stands for the standard SNPE-B method without taking any strategies we proposed.

Figure \ref{fig:valid_nsf} shows the results of the simulation experiments. The utilization of the ACK method effectively reduces LMD, which is predictable since the calibration kernel directly assigns weight to the sample pair $(\theta,x)$ based on the distance between $x_o$ and $x$. Nevertheless, it does not perform well on other metrics and still suffers from unstable learning. Surprisingly, the introduction of the calibration kernel can sometimes exacerbate this instability. DS and MISR indeed reduce the variance and thus narrow the error bar. However, they tend to lead to a loss in performance. There are two reasons for this occurrence. On the one hand, defensive distribution brings in irrelevant samples, making the training samples not sufficiently concentrated. On the other hand, MISR is prone to including `bad' samples from earlier rounds, which also results in lower sample quality.

As analyzed, the use of most of these methods individually does not yield an increased performance. This is because when one problem is addressed, another emerges. For instance, MISR and DS are able to reduce the variance but lead to low sample quality. On the other hand, ACK can help to concentrate the samples, but it may render estimators with large variances. However, combining all of our proposed strategies is amenable  to the SNPE-B method, which is denoted as `All-SNPE-B'.

Finally, PST causes a transformation of the target compared with TP. After incorporating this strategy, no significant negative impact on the performance was observed. Moreover, the efficiency of sampling can be improved. We thus prefer to apply PST rather than TP. Its improvement towards APT is reported in \citet{deistler2022truncated}. In some cases (especially M/G/1 in our experiment, see Figure \ref{fig:compare}), omitting PST may fail training (even with the repetition of 50 runs).

\subsection{Comparisons with other likelihood-free inference methods}\label{sec:exp_compare}
In this section, we provide a comparison between other existing likelihood-free inference methods, including SNPE-A, SNL, APT and our proposed All-SNPE-B. Additionally, to demonstrate the versatility of our proposed ACK technique, we propose the ACK-APT method, which is attained by directly plugging ACK as weight in the loss of APT. Since the original version of the APT and SNPE-B methods does not contain the PST technique, we also experiment with PST and TP versions of APT, All-SNPE-B, and ACK-APT on all tasks.
The results are presented separately in Figure \ref{fig:compare} and Figure \ref{fig:compare_appendix}.

First of all, we want to report that our proposed All-SNPE-B reduces computational costs (total iteration step times the number of samples used in each iteration) by nearly 6 times compared to APT as demonstrated in Table \ref{tab:cost}, while still maintaining comparable performance in most metrics and most tasks. Additionally, the All-SNPE-B method has a faster convergence rate in earlier rounds compared to other methods. Secondly, the ACK-APT method successfully outperforms APT in all cases as observed from the metrics, and with almost no additional computational cost, increasing the attractiveness of our methods. Thirdly, some pitfalls of metrics are observed in the case of SLCP (in panel B of Figure \ref{fig:valid_appendix}), where C2ST of All-SNPE-B indicates its failure compared to other methods while still maintaining the best LMD performance. This is due to the proposed ACK, which assigns a higher weight to those sample pair $(\theta,x)$ with $x$ close to $x_o$. Such behavior is favored by LMD and therefore exhibits superior performance. Visualizations of the marginal approximated posteriors are provided in Figure \ref{fig:post_compare}.

\begin{figure}[!t]
\hspace{0.02\textwidth}
\begin{subfigure}[t]{0.94\textwidth}
  \vspace{0pt}
    \includegraphics[width=1.00\textwidth]{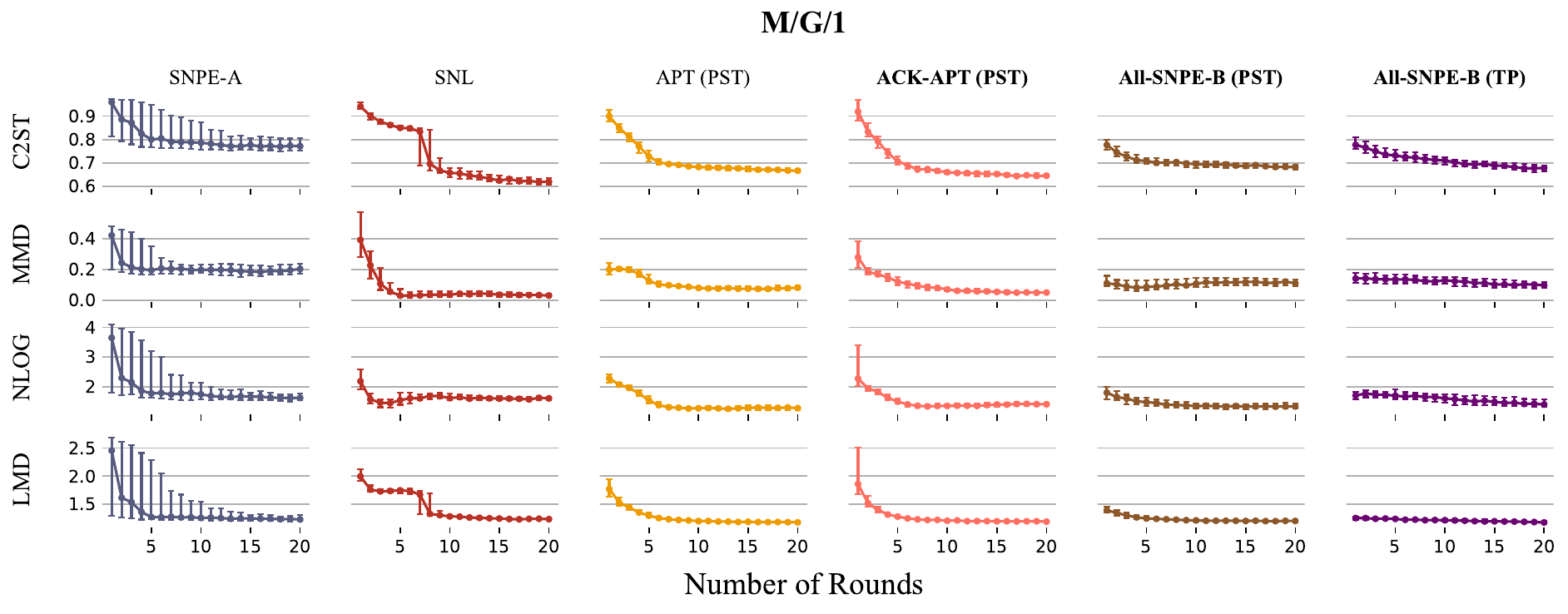}
\end{subfigure}
\hspace{-0.96\textwidth}
\begin{subfigure}[t]{0.02\textwidth}
  \vspace{5pt}
    \textbf{A}
\end{subfigure}

\hspace{0.02\textwidth}
\begin{subfigure}[t]{0.94\textwidth}
  \vspace{0pt}
    \includegraphics[width=1.00\textwidth]{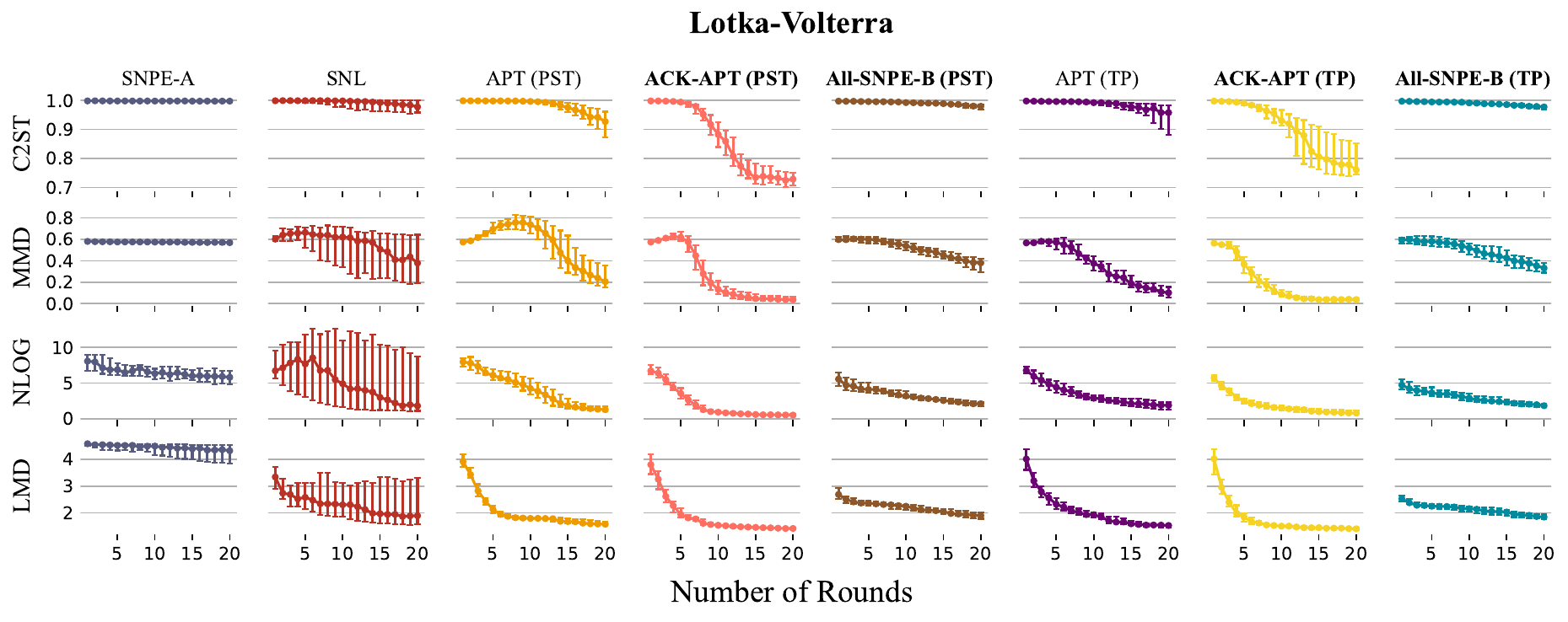}
\end{subfigure}
\hspace{-0.96\textwidth}
\begin{subfigure}[t]{0.02\textwidth}
  \vspace{5pt}
    \textbf{B}  
\end{subfigure}
\caption{\textbf{Our proposed strategies (All-SNPE-B and ACK-APT) versus other methods.} \textbf{A.} Performance on the M/G/1 queuing model. \textbf{B.} Performance on the Lotka-Volterra model. The horizontal axis represents the round of training and the error bars represent the mean with the upper and lower quarterlies.}
\label{fig:compare}
\end{figure}

\begin{table}[t]
\centering
\caption{Comparison of computational cost$^{1}$. The values represent the mean and standard deviation.}
\label{tab:cost}
\footnotesize
\begin{tabular}{@{}lccccc@{}}
\toprule
Task & M/G/1 & Lotka-Volterra & SLCP & Gaussian linear\\
\midrule
SNPE-B (TP) & 1.76 (0.26) & 1.62 (0.34) & 1.66 (0.51) & 1.92 (1.38) \\
SNPE-B (PST) & 1.64 (0.47) & 1.73 (0.37) & 1.74 (0.48) & 1.52 (0.29) \\
All-SNPE-B (TP) & 2.87 (0.24) & 2.42 (0.29) & 1.80 (0.23) & 1.12 (0.09) \\
All-SNPE-B (PST) & 1.95 (0.15) & 2.35 (0.29) & 1.85 (0.26) & 1.46 (0.12) \\
APT (PST) & 11.06 (1.02) & 12.96 (1.08) & 12.97 (1.01) & 10.04 (1.39) \\
ACK-APT (PST) & 10.78 (1.09) & 13.34 (1.08) & 12.71 (1.03) & 11.73 (1.22) \\
\bottomrule
\end{tabular}
\vspace{1ex}

{\raggedright $^{1}$The cost is quantified as the number of forward passes through the neural network per 1,000 instances.\par}
\end{table}
We refer the reader to Figure \ref{fig:post_compare} for a clear understanding of the advantages of All-SNPE-B in seeking modes. We assign this credit to ACK since it helps the neural network recognize the importance of samples according to the distance between $\mathbf{x}_o$ and its observation. We report a better approximation around the mode and the ground-truth distribution with the marginal plot. A similar result is also observed in the case of ACK-APT, which shows consistency instead of divergence. For example, Figure \ref{fig:post_compare}-B demonstrates the significant improvement of ACK over APT.

Lastly, our proposed method outperforms the traditional SMC-ABC method under a fixed computational budget. Figure \ref{fig:perf_smc} illustrates the outcomes of the SMC-ABC method applied to the two models under the same performance metrics. The results indicate that, after 20 iterations ($2\times 10^4$ simulations), our method achieves outcomes that are comparable to, or better than, those obtained by the SMC-ABC method after approximately $5 \times 10^5$ simulations.

\begin{figure}[htbp]
\hspace{0.14\textwidth}
\begin{subfigure}[t]{0.02\textwidth}
  \vspace{0pt}
    \textbf{A}
\end{subfigure}
\begin{subfigure}[t]{0.70\textwidth}
  \vspace{0pt}
    \includegraphics[width=1.00\textwidth]{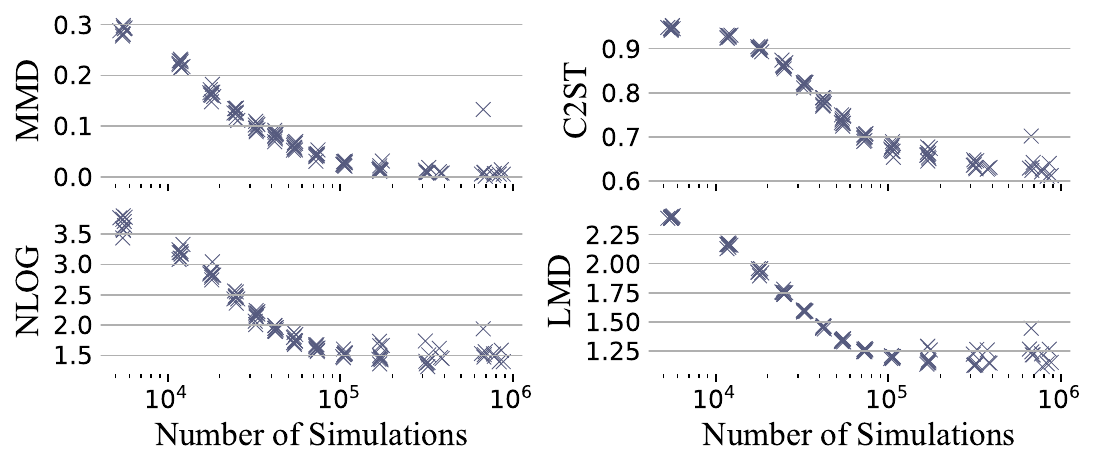}
\end{subfigure}
    
\hspace{0.14\textwidth}
\begin{subfigure}[t]{0.02\textwidth}
  \vspace{3pt}
    \textbf{B}  
\end{subfigure}
\begin{subfigure}[t]{0.70\textwidth}
  \vspace{0pt}
    \includegraphics[width=1.00\textwidth]{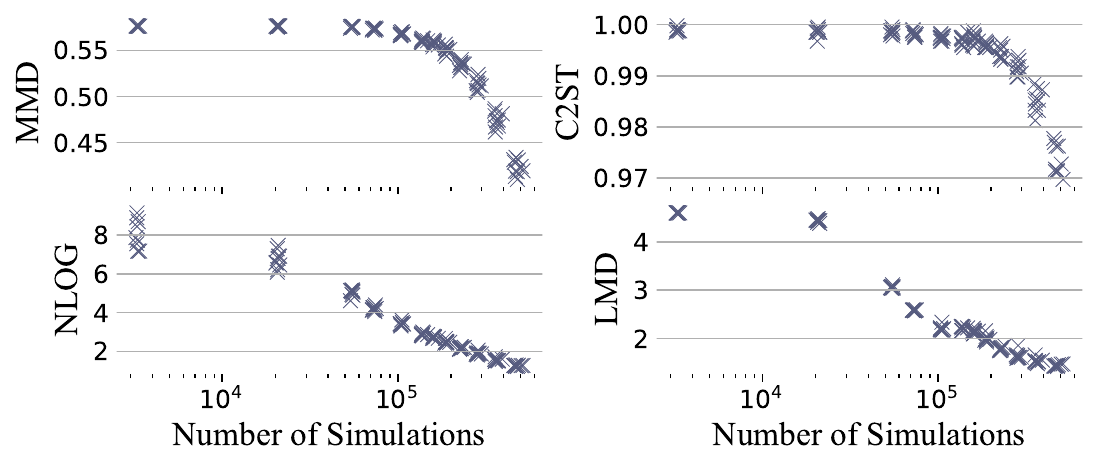}
\end{subfigure}
\caption{\textbf{Approximation accuracy by SMC-ABC method against the number of simulations.} \textbf{A.} Performance on the M/G/1 queuing model. \textbf{B.} Performance on the Lotka-Volterra model.}
\label{fig:perf_smc}
\end{figure}

\subsection{Inference on high-dimensional state-space model with real-world dataset.} \label{sec:exp_realistic}

We conduct an analysis on a high-dimensional time series dataset comprising Airbnb rental prices in Seattle, WA, USA, for the year 2016. The dataset, obtained from \textit{kaggle.com}, contains 928,151 entries spanning $T = 365$ days, with each record representing the rental price of an available listed property or space on a given day. As indicated by \citet{rodrigues2020likelihood}, the price distribution for each day $t$, denoted as $y_1^{(t)}, y_2^{(t)}, \cdots, y_{n_t}^{(t)}$, is non-Gaussian even after transformation. To model this distribution, we adopt a flexible g-and-k distribution \citep{rayner2002numerical}, which can effectively capture the complex characteristics of the data. Although the g-and-k distribution does not have a closed-form density function, it can be described through its quantile function as follows:
\begin{equation}\label{eq:gandk_quantile}
Q(q|\beta) = A + B \left[ 1 + 0.8 \frac{1 - \exp\{-g z(q)\}}{1 + \exp\{-g z(q)\}} \right] \left( 1 + z(q)^2 \right)^k z(q),
\end{equation}
where $\beta = (A, B, g, k)^\top$ is the model parameter, with $B > 0$ and $k > -1/2$, and $z(q) = \Phi^{-1}(q)$ represents the inverse CDF of a standard normal distribution $\mathcal{N}(0, 1)$. Notably, when both $g$ and $k$ are set to zero, the distribution reduces to a normal distribution.

Following the approach of \citet{rodrigues2020likelihood}, we employ an intractable nonlinear state-space model to fit the observed data:
\begin{align*}\label{eq:price_sys_eq}
y_1^{(t)},y_2^{(t)},\cdots,y_{n_t}^{(t)} &\sim p(y \mid \beta_t), \\
\varphi(\beta_t) &= \lambda_t = F_t^\top \theta_t, \\
\theta_{t+1} &= G_t \theta_t + w_t,
\end{align*}
where $y_i^{(t)}$ denotes the standardized price observed at time $t$ and $\lambda_t = (\lambda_{1,t},\dots, \lambda_{4, t})^\top$. The function $\varphi$ serves as the link function, and $\varphi^{-1}(\lambda_t) = (\lambda_{1, t}, \exp(\lambda_{2, t}), \lambda_{3, t}, \exp(\lambda_{4, t}) - 0.5)^\top = \beta_t$ represents the four parameters of the g-and-k distribution. The matrix $F_t \in \mathbb{R}^{36 \times 4}$ is the design matrix that maps the state vector $\theta_t$ to the linear predictor $\lambda_t$, while the state vector $\theta_t \in \mathbb{R}^{36\times 1}$ corresponds to the constant term, local linear trend, weekly seasonal effect, and summer effect for each component in $\beta_t$. The evolution matrix $G_t \in \mathbb{R}^{36 \times 36}$ dictates the system's dynamics, and $w_t \sim N(0, W_t)$ represents the process noise. Detailed specifications of the model are provided in Appendix \ref{appendix:model_detail}.

To reduce the output dimensionality, we take a summary statistic $s_t = \varphi(\hat{\beta}_t)$ for each day, where $\hat{\beta}_t$ is the $L$-moments estimator of $\beta_t$, derived from the daily observations $y_1^{(t)}, y_2^{(t)}, \cdots, y_{n_t}^{(t)}$. Although the $L$-moments estimator is not fully sufficient, it is highly informative and provides an almost unbiased estimate across all sample sizes and parameter settings \citep{peters2016estimating}. Consequently, the summary statistics for the entire dataset have a dimensionality of $4T$, which, in the case of the rental price data for one year, amounts to a total dimension of 1460.

In Figure \ref{fig:state_model_perf}-A, we present the LMD values for several methods, including SNPE-A, SNPE-B, APT, as well as our proposed All-SNPE-B and ACK-APT. We did not include SNL in this evaluation due to the high dimensionality of the observations, which makes MCMC sampling impractical during each iteration. The results demonstrate that methods based on the ACK strategy significantly reduce the LMD values, and our proposed All-SNPE-B method consistently achieves the lowest LMD values across all methods. 
We emphasize that importance sampling often suffers from performance degradation in high-dimensional problems \citep{owen2000safe}, but the MISR strategy mitigates this to some extent by increasing the sample size, thereby reducing variance in high-dimensional settings. In the case of APT, while it avoids incorporating a direct density ratio in the loss function, approximating its normalization constant $Z(x, \phi)$ (which is an integral over the parameter space $\Theta$) in Eq.\eqref{eq:loss_apt} using $M \ll N$ atoms can introduce higher variance, particularly in high-dimensional cases when only a limited number of samples are used. We also report the posterior marginal density plots over 10 runs in Figure \ref{fig:state_marginal}-A.

Finally, we validate our findings with simulated data. We fix the true model parameters and generate observations $x_o$ accordingly, enabling the calculation of the NLOG metric. The simulation period was set to 4 weeks, with $T = 28$, resulting in a summary statistic dimensionality of $d = 112$. In Figure \ref{fig:state_model_perf}-B, we present the NLOG and LMD values for each method. Additionally, we report the posterior marginal density plots over 10 runs in Figure \ref{fig:state_marginal}-B. Our method successfully captures the true model parameters, whereas other methods fail (notably, including APT) in this scenario.
The results show that the All-SNPE-B method achieves consistently remarkable performance, while the ACK-APT method, though less stable, still outperforms the other baseline methods.
The challenges faced by SNPE-B and APT arise from the high dimensionality of the sample space, even after incorporating summary statistics. In such settings, conditional density estimators struggle to focus on the sparse and highly informative training samples near $x_o$. The ACK strategy addresses this by prioritizing these critical regions, rather than attempting to approximate the entire posterior distribution across the full parameter space.

\begin{figure}[tp]
\hspace{0.02\textwidth}
\begin{subfigure}[t]{0.96\textwidth}
  \vspace{0pt}
    \includegraphics[width=1.00\textwidth]{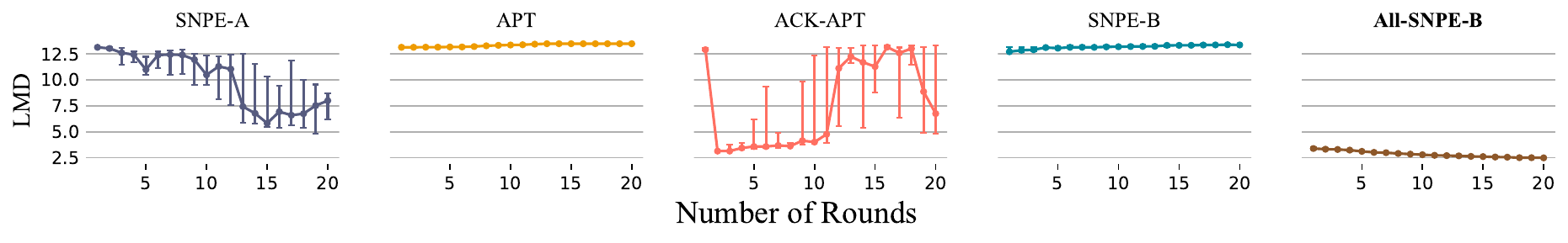}
\end{subfigure}
\hspace{-0.99\textwidth}
\begin{subfigure}[t]{0.02\textwidth}
  \vspace{3pt}
    \textbf{A}  
\end{subfigure}

\hspace{0.02\textwidth}
\begin{subfigure}[t]{0.96\textwidth}
  \vspace{0pt}
    \includegraphics[width=1.00\textwidth]{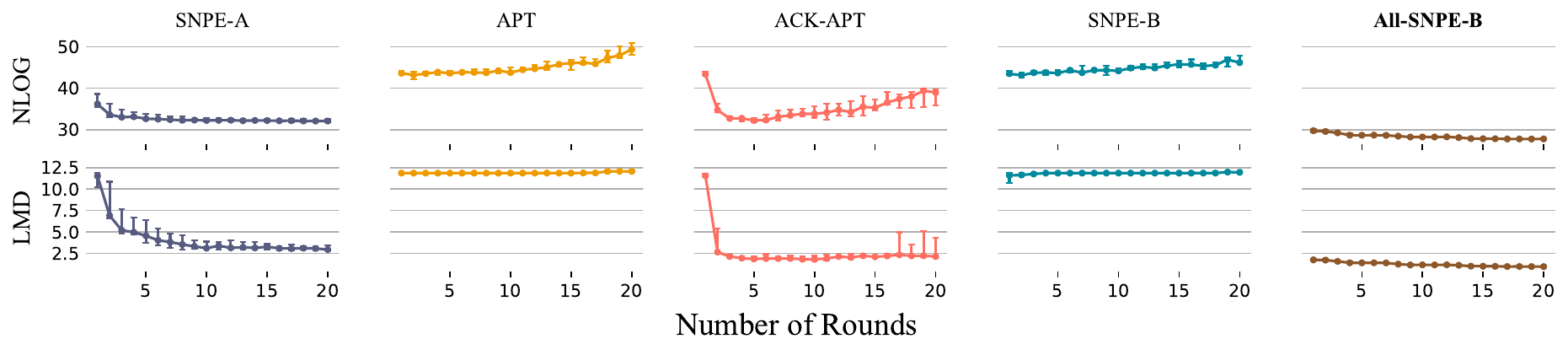}
\end{subfigure}
\hspace{-0.99\textwidth}
\begin{subfigure}[t]{0.02\textwidth}
  \vspace{3pt}
    \textbf{B}  
\end{subfigure}

\caption{Performance of different methods on the state-space model. \textbf{A.} Inference on one year of Seattle rental price data ($T = 365$). \textbf{B.} Inference on four weeks of simulated data ($T = 28$).
The horizontal axis represents the round of training and the error bars represent the mean with the upper and lower quarterlies.}
\label{fig:state_model_perf}
\end{figure}

\begin{figure}[t]
\hspace{0.00\textwidth}
\begin{subfigure}[t]{0.02\textwidth}
  \vspace{1pt}
    \textbf{A}  
\end{subfigure}
\begin{subfigure}[t]{0.94\textwidth}
  \vspace{0pt}
    \includegraphics[width=\textwidth]{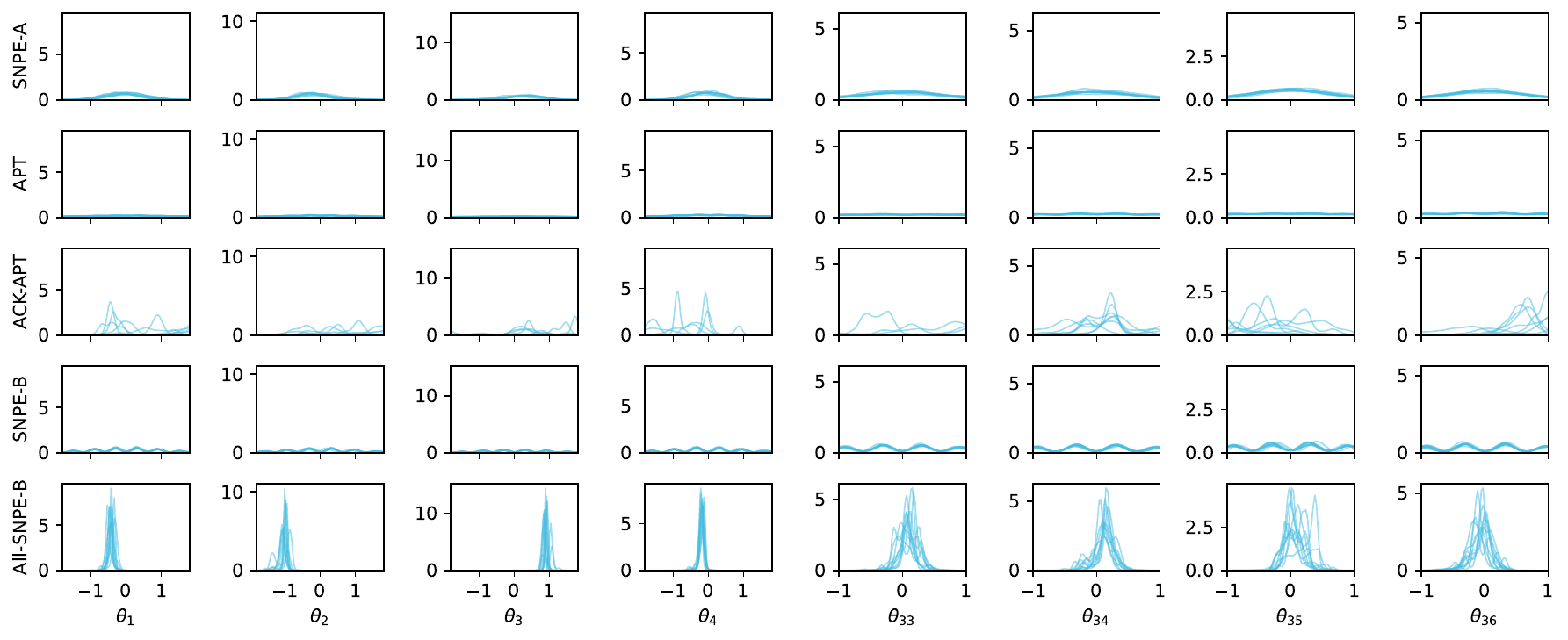}
\end{subfigure}

\hspace{0.00\textwidth}
\begin{subfigure}[t]{0.02\textwidth}
  \vspace{5pt}
    \textbf{B}  
\end{subfigure}
\begin{subfigure}[t]{0.94\textwidth}
  \vspace{0pt}
    \includegraphics[width=\textwidth]{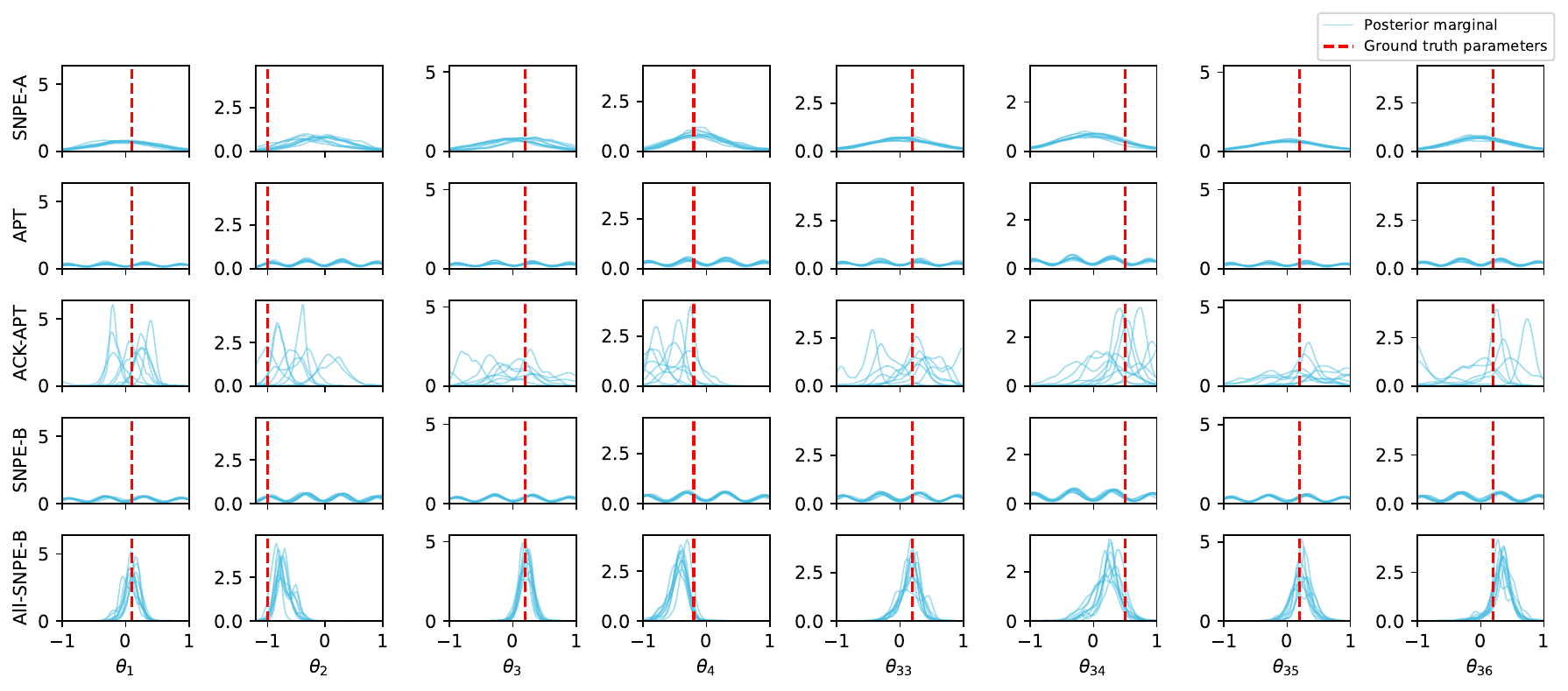}
\end{subfigure}

\caption{\textbf{Comparison of posterior densities evaluating at $x_o$. } 
\textbf{A.} Inference on one year of Seattle rental price data ($T=365$). 
\textbf{B.}  Inference on four weeks of simulated data ($T=28$). 
Our proposed methods result in valid posterior information under limited simulations.}
\label{fig:state_marginal} 
\end{figure}

\section{Discussion} \label{section:discuss}

This paper mainly focuses on several variance reduction techniques inspired by classic importance sampling methods, such as defensive sampling and multiple importance sampling. In contrast to the SNPE-A and APT methods, the performance of the SNPE-B method may sometimes be sub-optimal due to the additional density ratio attached to its loss function, resulting in a larger variance of the loss function and its gradient during the training process. Our experiments indicate that the large variance makes the calibration kernel inapplicable, leading to low data efficiency. The low data efficiency could lead to inferior approximation of the density estimators. The SNPE-B method achieves significant improvement by simultaneously applying variance reduction techniques and ACK methods.

Since most of the NPE methods can be finally broken down to the minimization of the KL divergence, the strategies we have proposed are transferable to other methods. For instance, the ACK strategy can be extended to the APT method and SNRE method to improve the training in the early rounds. The PST strategy can also be applied to the APT method to avoid the issue of density leakage. Unlike the SNPE-B method, APT requires all $\theta$ to be within the support of the prior distribution. Introducing the PST method can avoid additional steps related to truncation through rejection sampling and estimating the normalization factor. PST is necessary in cases like M/G/1, where TP can lead to failure of convergence, while in cases like the Lotka-Volterra model and SLCP models, incorporating PST can provide extra benefit in performance.
Notably, ACK demonstrates impressive ability in helping the density estimator to converge correctly, which coincides with our motivation for designing it. We also performed our tests in a high dimensional state-space model where the widely used APT fails, demonstrating the superiority of our proposed method and the potential of SNPE-B. 

It is worth mentioning that the proposed form of calibration kernels in SNPE-B \cite{lueckmann2017flexible}, is adaptive since it depends on the Gaussian proposal distribution, which updates with rounds. Our proposed method is also adaptive since the calibration kernel rate updates according to solving an equation related to ESS. 

\section*{Acknowledgments}
We greatly appreciate two anonymous referees' valuable suggestions and comments on improving this paper. We thank Professor Jakob Macke for the helpful discussions.
The work of the first author and third authors was supported by the National Natural Science Foundation of China (No.12171454, U19B2940) and Fundamental Research Funds for the Central Universities. The
work of the fourth author was supported by the National Natural Science Foundation of China (No.12071154), the Guangdong Basic and Applied Basic Research Foundation (No.2024A1515011876).

\bibliography{refs}

\appendix

\section{Proofs of Theorem \ref{thm:calib_analysis}} \label{appendix:calib_analysis}

\begin{proof}
The effect of $\tau$ can be deduced from the following equations
\begin{equation*}\begin{aligned}
\mu(\tau) & =\mathbb{E}_{\tilde{p}(\theta, x)}\left[K_\tau\left(x, x_o\right) g(\theta, x)\right]\\
& =\int K_\tau\left(x, x_o\right)\left(\int g(\theta, x) \tilde{p}(\theta, x) \mrd \theta\right) \mrd x\\
& =\int K_\tau\left(x, x_o\right) h_g(x) \mrd x \\
& =\int(2 \pi)^{-d / 2} |\Sigma|^{-1/2} \tau^{-d} \exp \left(-(x-x_o)^\top\Sigma^{-1}(x-x_o) /\left(2 \tau^2\right)\right) h_g(x) d x \\
& =\int(2 \pi)^{-d / 2} |\Sigma|^{-1/2} \exp \left(-x^\top\Sigma^{-1} x / 2\right) h_g\left(x_o+\tau x\right) \mrd x \\
& =\int(2 \pi)^{-d / 2} |\Sigma|^{-1/2} \exp \left(-x^\top\Sigma^{-1} x / 2\right)\left(h_g\left(x_o\right)+\tau x^T \nabla_x h_g\left(x_o\right)+\mathcal{O}\left(\tau^2\right)\right) \mrd x \\
& =h_g\left(x_o\right)+\mathcal{O}\left(\tau^2\right),
\end{aligned}\end{equation*}
where $h_g(x):=\int g(\theta, x)\tilde p(\theta, x)\mrd\theta$. On the other hand, 
\begin{equation*}\begin{aligned}
V(\tau) & =\operatorname{Var}_{\tilde{p}(\theta, x)}\left[K_\tau\left(x, x_o\right) g(\theta, x)\right] \\
& =\mathbb{E}_{\tilde{p}(\theta, x)}\left[K_\tau^2\left(x, x_o\right) g^2(\theta, x)\right]-\left(\mathbb{E}_{\tilde{p}(\theta, x)}\left[K_\tau\left(x, x_o\right) g(\theta, x)\right]\right)^2 \\
& =\int K_\tau^2\left(x, x_o\right) h_{g^2}(x) d x-\left(h_g\left(x_o\right)+\mathcal{O}\left(\tau^2\right)\right)^2 \\
& =\int(2 \pi)^{-d}  |\Sigma|^{-1} \tau^{-2 d} \exp \left(-(x-x_o)^\top\Sigma^{-1}(x-x_o) / \tau^2\right) h_{g^2}(x) \mrd x-\left(h_g\left(x_o\right)+\mathcal{O}\left(\tau^2\right)\right)^2\\
& =(2 \pi)^{-d / 2} 2^{-d / 2} |\Sigma|^{-1/2} \tau^{-d} \int(2 \pi)^{-d / 2} |\Sigma|^{-1/2} \exp \left(-x^\top\Sigma^{-1}x / 2\right) \\
& \quad\left(h_{g^2}\left(x_o\right)+\sqrt{2} \tau x^T \nabla_x h_{g^2}\left(x_o\right)+\mathcal{O}\left(\tau^2\right)\right) \mrd x -\left(h_g\left(x_o\right)+\mathcal{O}\left(\tau^2\right)\right)^2 \\
& =(2 \pi)^{-d / 2} 2^{-d / 2} |\Sigma|^{-1/2} \tau^{-d} h_{g^2}\left(x_o\right)+\mathcal{O}\left(\tau^{2-d}\right)-\left(h_g\left(x_o\right)+\mathcal{O}\left(\tau^2\right)\right)^2 \\
& =C h_{g^2}\left(x_o\right) \tau^{-d}+\mathcal{O}\left(\tau^{2-d}\right),
\end{aligned}\end{equation*}
where $C=(2 \pi)^{-d / 2} 2^{-d / 2}|\Sigma|^{-1/2}$.
\end{proof}

\section{Summary statistics-based dimensional reduction} \label{appendix:summary_stat}
Using summary statistics enables us to reduce the dimension of the data being processed, by calculating relevant features, which in turn limits the impact of high-dimensional features that may be noisy or irrelevant. By doing so, we can mitigate the adverse effects of the curse of dimension and achieve more efficient and effective computations. The kernel function that incorporates the summary statistics is then
\begin{equation*}
K^S_\tau(x, x_o)=(2\pi)^{-d_S/2}\tau^{-d_S}\exp\left(-(S(x)-S_o)^\top\Sigma_S^{-1}(S(x)-S_o)/(2\tau^2)\right),
\end{equation*}
where $d_S=\mathrm{dim}(S)\ll d=\mathrm{dim}(x)$ and $\Sigma_S$ is estimated by $\hat\Sigma_S=\frac{1}{N-1}\sum_{i=1}^N(S(x_i)-\bar{S})(S(x_i)-\bar{S})^\top$, $\bar S=\frac{1}{N}\sum_{i=1}^N S(x_i)$ and $S_o=S(x_o)$. If the summary statistics are used, we also can use $p(\theta|S_o)$ to present $p(\theta|x_o)$, and both are equivalent when $S=S(x)$ is sufficient for $\theta$. The corresponding algorithm is given in Algorithm \ref{alg:summary_1}.
\begin{algorithm}
\caption{Summary statistics-based dimensional reduction method}\label{alg:summary_1}
\begin{algorithmic}[1]
\State Initialization: $\tilde p(\theta) := p(\theta)$, given $\alpha\in(0, 1)$, number of rounds $R$, simulations per round $N$, given summary statistics $S(x)$, setting $S_o=S(x_o)$
\For{$r=1,2,\cdots,R$}
\State sample $\{\theta_i\}_{i=1}^N$ from $\tilde p(\theta)$
\State sample $x_i\sim p(x|\theta_i),i=1,2,\cdots,N$, resulting in $\{x_i\}_{i=1}^N$
\State update $\phi^* = \mathop{\arg\min}_{\phi} -\frac{1}{N}\sum_{i=1}^N\frac{p(\theta_i)}{\tilde p(\theta_i)} K_\tau^S (x_i, x_o) \log q_{F(S(x_i), \phi)}(\theta_i)$
\State set $\tilde p(\theta):=(1-\alpha) q_{F(S(x_o), \phi^*)}(\theta) + \alpha p_{\mathrm{def}}(\theta)$
\EndFor \\
\Return $q_{F(x_o, \phi^*)}(\theta)$
\end{algorithmic}
\end{algorithm}

\section{Additional model information} \label{appendix:model_detail}

This section provides detailed information on the SLCP model, Gaussian linear model, and g-and-k based state-space model.

\textbf{SLCP model \cite{papamakarios2019sequential}.} 
The model has a complex and multimodal posterior form. Model parameters $\theta=(\theta_1,\cdots,\theta_5)$ obeying a uniform prior on a bounded rectangle as $\mathcal{U}(-3, 3)^5$. The dimension of data $x$ is 8, satisfying $p(x|\theta)=\prod_{i=1}^4 \mathcal{N}(x_{(2i-1, 2i)} | \mu(\theta), \Sigma(\theta))$, which is a two-dimensional normal distribution of four samples concatenated. The mean is defined by $\mu(\theta)=(\theta_1, \theta_2)^\top$ and the covariance matrix takes the form
\begin{equation*}
\Sigma(\theta)= \left(\begin{array}{cc}
s_1^2 & \rho s_1 s_2 \\
\rho s_1 s_2 & s_2^2 \\
\end{array}\right),
\end{equation*}
where $s_1=\theta_3^2$, $s_2=\theta_4^2$ and $\rho=\tanh(\theta_5)$.
In our experiments, the ground truth parameters are
\begin{equation*}
\theta^* = (0.7,  -2.9, -1, -0.9, 0.6),
\end{equation*}
and the observed data $x_o$ simulated from the model with ground truth parameters $\theta^*$ are
\begin{equation*}
x_o = (1.4097,\ -1.8396,\  0.8758,\ -4.4767,\ -0.1753,\ -3.1562,\ -0.6638,\ -2.7063).
\end{equation*}
The standard deviation of each component of the data $x$ generated from ground truth parameters $\theta^*$ are
\begin{equation*}
s = (1,\ 0.81,\ 1,\ 0.81,\ 1,\ 0.81,\ 1,\ 0.81),
\end{equation*}

\textbf{Gaussian linear model \citep{lueckmann2021benchmarking}.} This model utilizes a 10-dimensional Gaussian model configuration where the parameter $\theta$ represents the mean, and the covariance is fixed at $0.1\times I_{10}$. This configuration allows the examination of how algorithms manage trivial scaling of dimensionality \citep{lueckmann2021benchmarking}. The prior for $\theta$ is defined by a bounded uniform prior $\mathcal{U}(-1, 1)^{10}$. The ground truth parameters are set as
\begin{align*}
\theta^* = (-&0.9527,\ -0.1481,\ 0.9824,\ 0.4132,\ 0.9904,\\ -&0.7402,\ 0.7862,\ 0.0437,\ -0.6261,\ -0.7651),
\end{align*}
and the observed data $x_o$ simulated from the model with these parameters are
\begin{align*}
x_o = (-&0.5373,\ -0.2386,\ 0.8192,\ 0.6407,\ 0.4161,\\ -&0.0974,\ 1.1292,\ -0.0584,\ -0.9705,\ -0.9423).
\end{align*}
The standard deviation $s$ for each component of the data $x$ generated from the ground truth parameters is $\sqrt{1/10}$.

\textbf{State-space model \citep{rodrigues2020likelihood}.} This model provides a flexible framework for fitting time series datasets. Following \citet{rodrigues2020likelihood}, each parameter of the $g$-and-$k$ distribution, $\beta_t = (\beta_t^{(1)}, \ldots, \beta_t^{(4)})^\top$, is represented by its own set of state parameters, $\theta_t^{(i)}$, $i=1,\dots,4$. The system matrices, $F_t^{(i)}$ and $G_t^{(i)}$, are defined as follows:
\begin{equation*}
F_t^{(i)} = (E_2, E_6, \delta(t))^\top, \quad G_t^{(i)} =
\begin{pmatrix}
J_2 & 0_{2 \times 6} & 0_{2 \times 1} \\
0_{6 \times 2} & P_6 & 0_{6 \times 1} \\
0_{1 \times 2} & 0_{1 \times 6} & 1
\end{pmatrix}.
\end{equation*}
where 
\begin{equation*}
J_2 = \begin{pmatrix}
1 & 10^{-3} \\
0 & 1
\end{pmatrix}, \quad P_6 = \begin{pmatrix}
-1_{1 \times 5} & -1 \\
I_5 & 0_{5 \times 1}
\end{pmatrix},
\end{equation*}
$E_n = (1, 0, \cdots, 0)$ is a $n$-dimensional vector, and $\delta(t)$ is an indicator function epresenting the summer season. The evolution matrix $G_t^{(i)}$ comprises several components: a block $J_2$ to model the local linear trend of the latent level, a permutation matrix $P_6$ to capture the weekly seasonal effect, and specific terms to describe the summer effect using $\theta_{9,t}^{(i)}$. The complete model is constructed by combining $F_t^{(i)}$ and $G_t^{(i)}$ with the Kronecker product, resulting in $F_t = F_t^{(i)} \otimes I_4$ and $G_t = G_t^{(i)} \otimes I_4$, with $\theta_t = (\theta_t^{(1)}, \ldots, \theta_t^{(4)})$. The process noise is defined as $w_t\sim N (0, 10^{-5} I_{36})$, and the prior of the initial state is set as $\theta_0\sim N(0, 10^{-1} I_{36})$.

For obtaining the simulated dataset, we set the ground truth parameters as
\begin{align*}
\theta^* = (
&0.1,\ 1.0,\ 0.2,\ -0.2,\ 0.1,\ 0.1,\ -0.05,\ -0.05,\ 0.1,\ 0.1,\ 0.1,\ 0.1,\ 0.05,\ 0.05,
\\&0.05,\ 0.05,\ 0.01,\ 0.01,\ 0.01,\ 0.01,\ -0.01,\ -0.01,\ -0.01,\ -0.01,\ 0.03,\ -0.03,\ 
\\&-0.03,\ -0.03,\ -0.05,\ -0.05,\ -0.05,\ -0.05,\ 
0.2,\ 0.5,\ 0.2,\ 0.2).
\end{align*}

\begin{figure}[tp]

\begin{subfigure}[t]{1.00\textwidth}
\vspace{0pt}
    \includegraphics[width=1.00\textwidth]{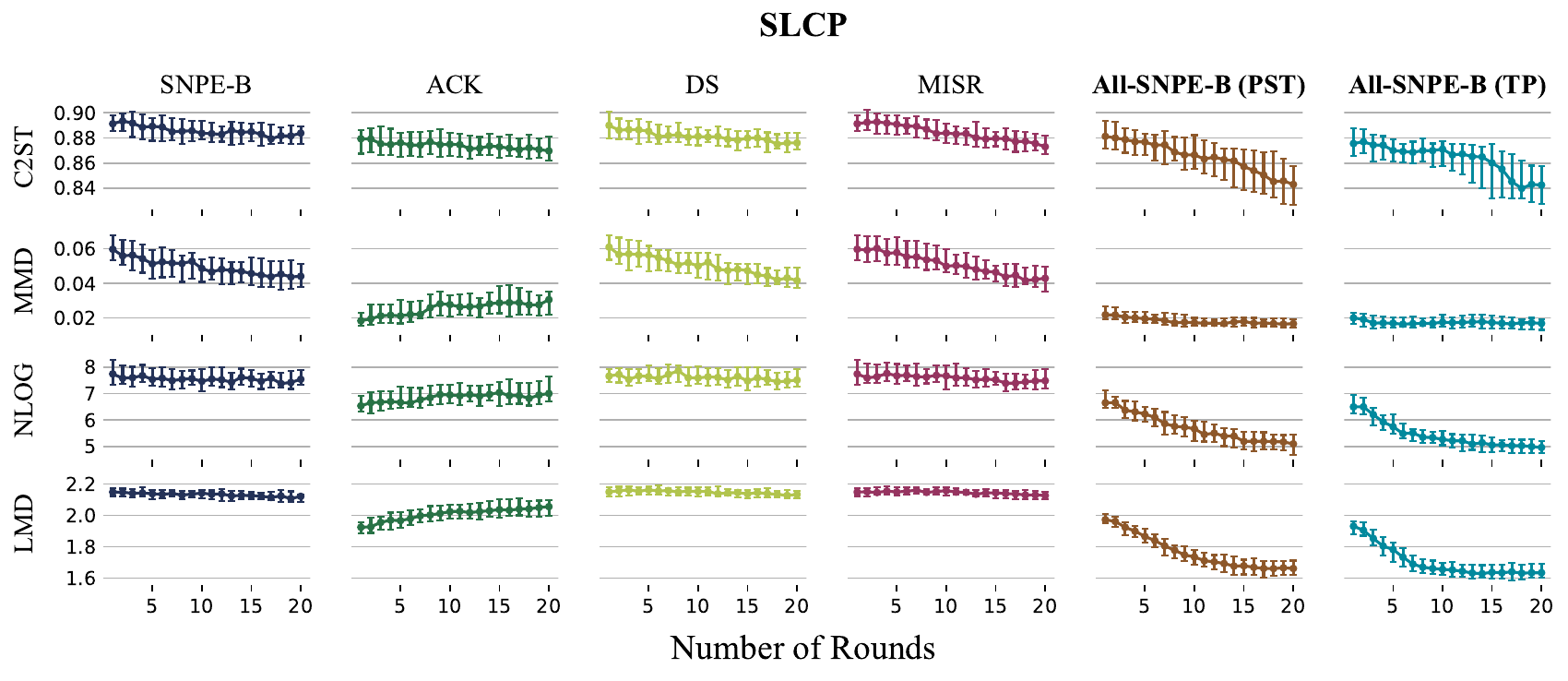}
\end{subfigure}
\hspace{-1.01\textwidth}
\begin{subfigure}[t]{0.02\textwidth}
  \vspace{5pt}
    \textbf{A}  
\end{subfigure}

\begin{subfigure}[t]{1.00\textwidth}
\vspace{0pt}
    \includegraphics[width=1.00\textwidth]{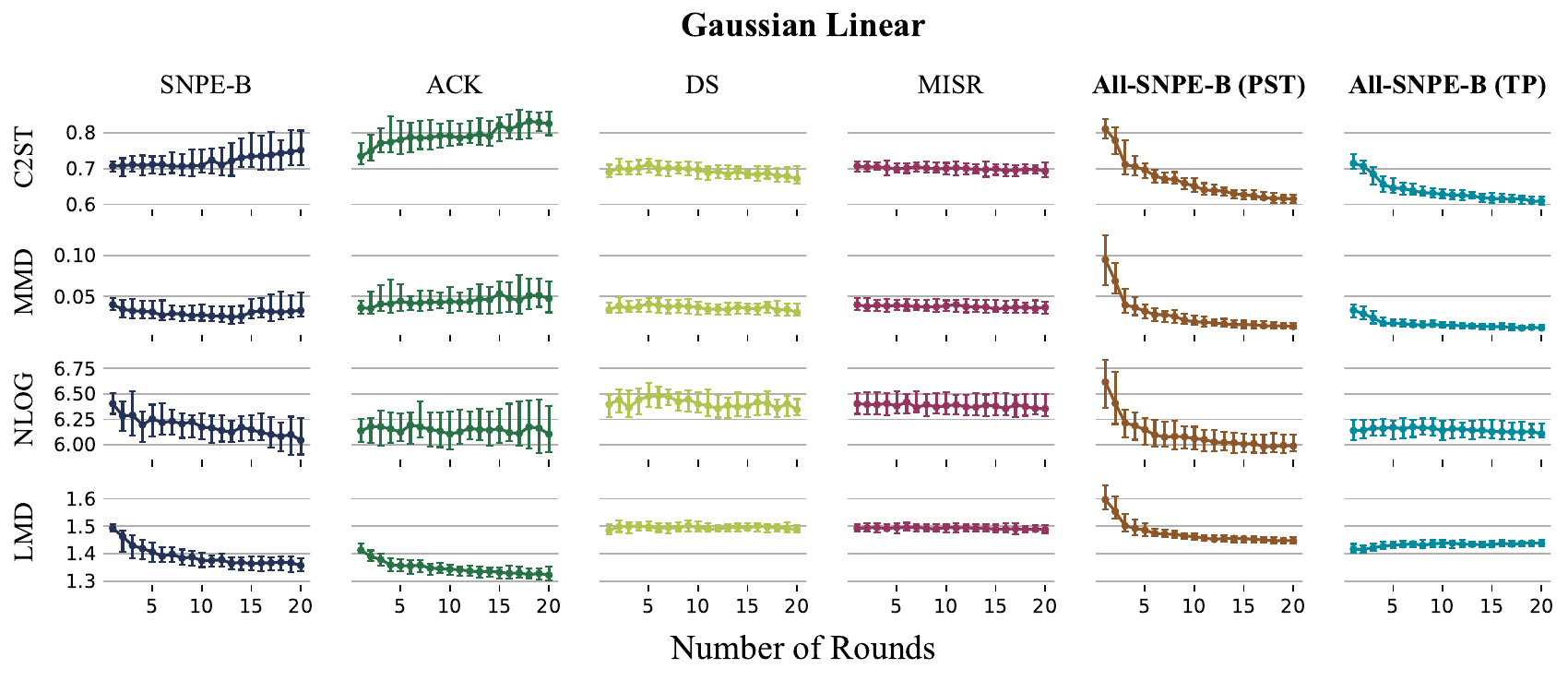}
\end{subfigure}
\hspace{-1.01\textwidth}
\begin{subfigure}[t]{0.02\textwidth}
  \vspace{5pt}
    \textbf{B}  
\end{subfigure}

\caption{\textbf{Simulation experiments on proposed strategies.} \textbf{A.} Performance on the SLCP model. \textbf{B.} Performance on the Gaussian linear model. The horizontal axis represents the round of training and the error bars represent the mean with the upper and lower quartiles.}
\label{fig:valid_appendix} 
\end{figure}

\begin{figure}[tp]
\hspace{0.02\textwidth}
\begin{subfigure}[t]{0.94\textwidth}
  \vspace{0pt}
    \includegraphics[width=1.00\textwidth]{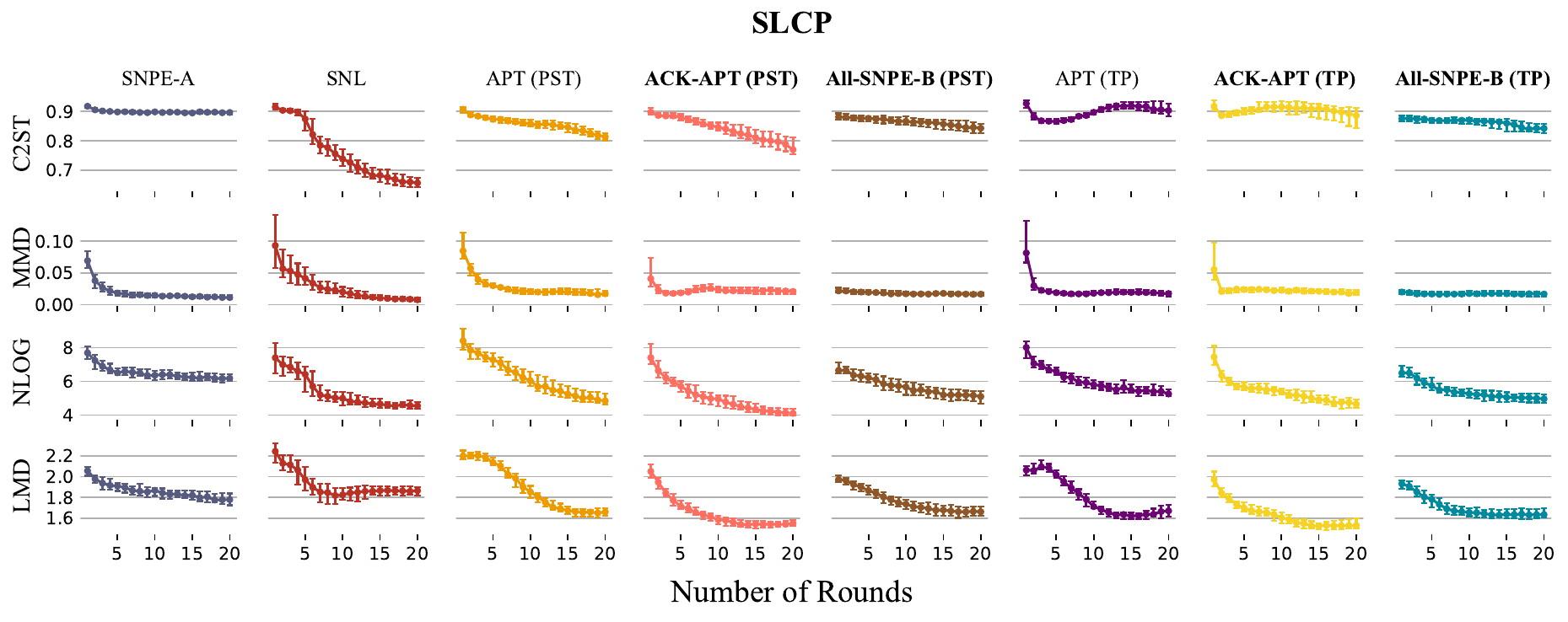}
\end{subfigure}
\hspace{-0.96\textwidth}
\begin{subfigure}[t]{0.02\textwidth}
  \vspace{5pt}
    \textbf{A}  
\end{subfigure}

\hspace{0.02\textwidth}
\begin{subfigure}[t]{0.94\textwidth}
  \vspace{0pt}
    \includegraphics[width=1.00\textwidth]{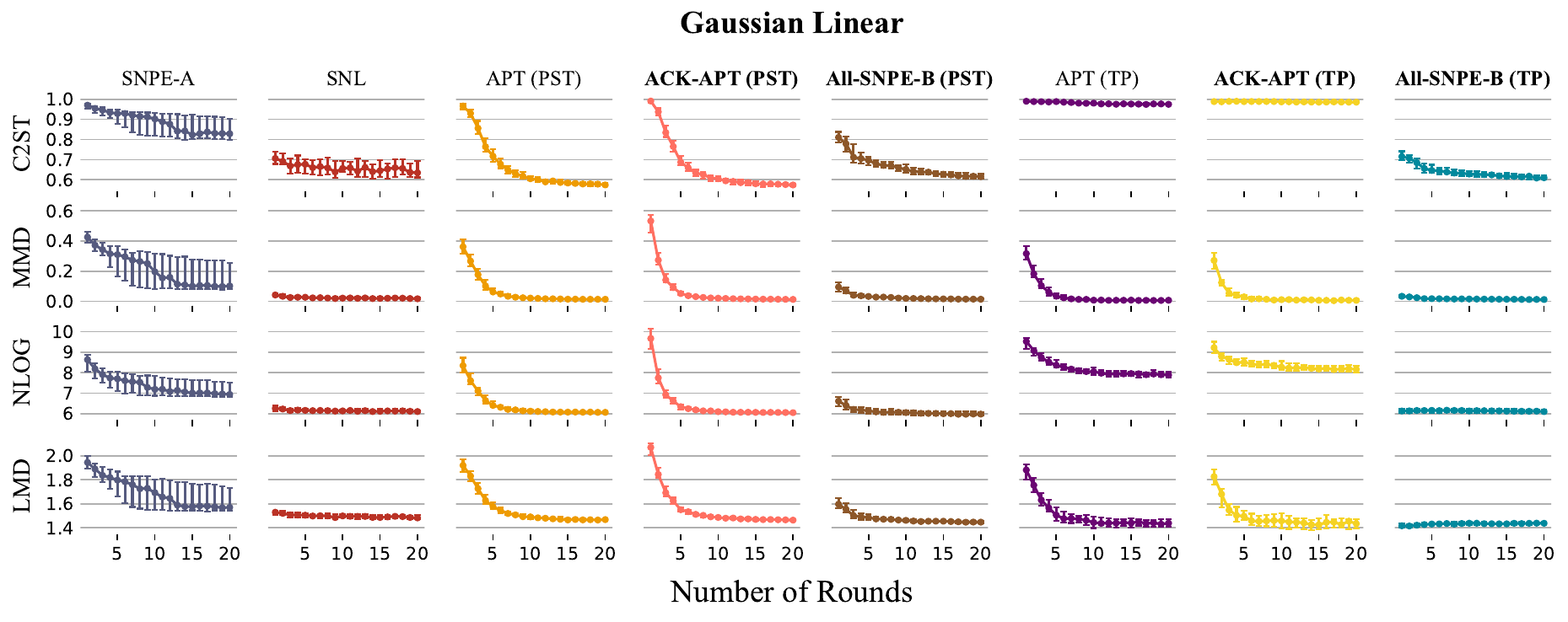}
\end{subfigure}
\hspace{-0.96\textwidth}
\begin{subfigure}[t]{0.02\textwidth}
  \vspace{5pt}
    \textbf{B}  
\end{subfigure}
\caption{\textbf{Our proposed strategies (All-SNPE-B and ACK-APT) versus other methods.}  
\textbf{A.} Performance on the SLCP model. 
\textbf{B.} Performance on Gaussian linear model. 
The horizontal axis represents the round of training and the error bars represent the mean with the upper and lower quarterlies.}
\label{fig:compare_appendix}
\end{figure}

\section{Additional experimental results} \label{appendix:exp_detail}
This section provides the visualization of posterior estimates with our proposed method on the four models.
In Figure \ref{fig:post_compare}, we showcase the approximate posterior marginal density plots generated by the SMC-ABC method, which serves as a reference posterior for evaluating the accuracy of other methods. We compare them with the original SNPE-B and APT methods, as well as our proposed improved methods (All-SNPE-B and ACK-APT), all implemented using the same random seed settings. The results demonstrate that our proposed methods effectively capture the high-density regions of the posterior distribution. Furthermore, in comparison to the original methods, our approaches provide a more precise approximation of the posterior distribution. Notably, the original SNPE-B method fails to accurately approximate the posterior marginal density for the SLCP and Lotka-Volterra models.
\begin{figure}[htbp]
\hspace{0.02\textwidth}
\begin{subfigure}[t]{0.02\textwidth}
  \vspace{1pt}
    \textbf{A}  
\end{subfigure}
\begin{subfigure}[t]{0.9\textwidth}
  \vspace{0pt}
    \includegraphics[width=\textwidth]{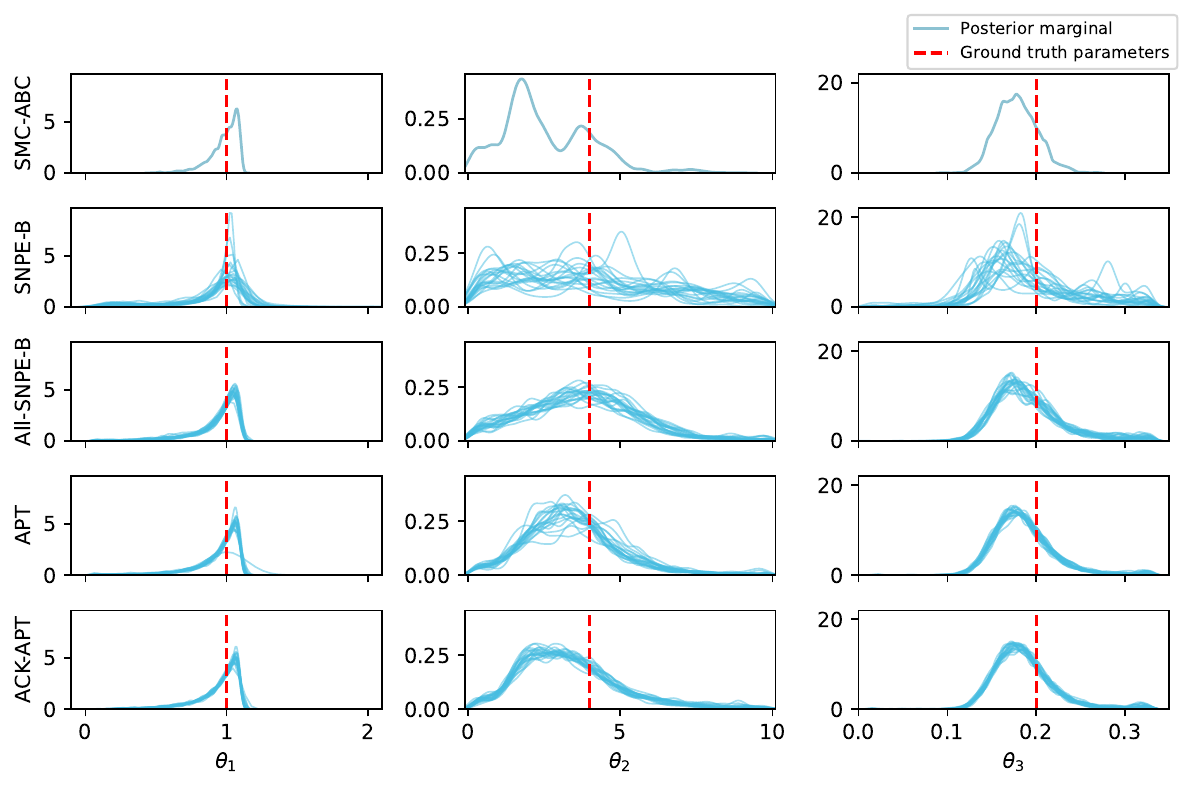}
\end{subfigure}

\hspace{0.02\textwidth}
\begin{subfigure}[t]{0.02\textwidth}
  \vspace{1pt}
    \textbf{B}  
\end{subfigure}
\begin{subfigure}[t]{0.9\textwidth}
  \vspace{0pt}
    \includegraphics[width=\textwidth]{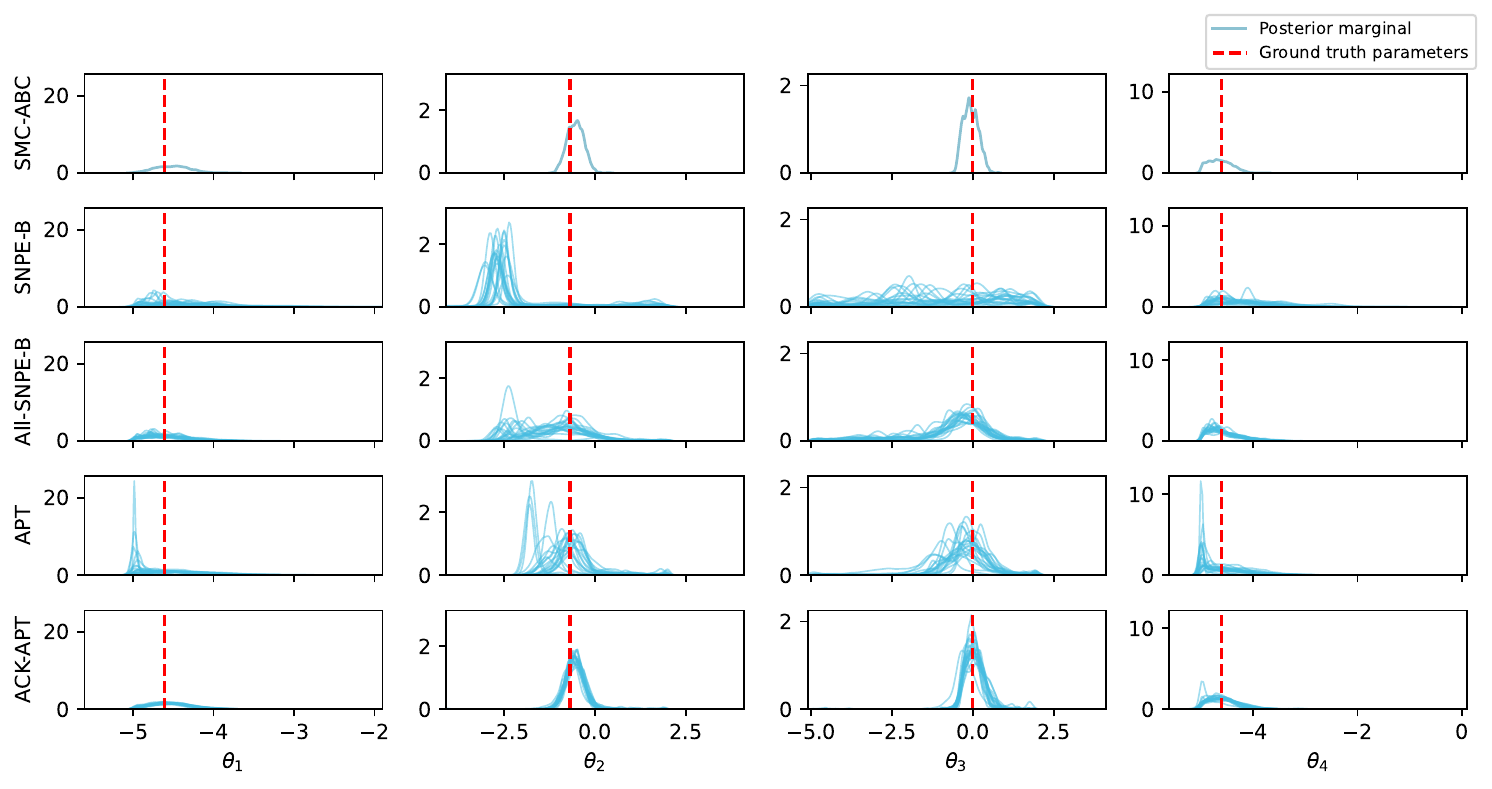}
\end{subfigure}
\end{figure}

\setcounter{figure}{8}
\begin{figure}
\hspace{0.02\textwidth}
\begin{subfigure}[t]{0.02\textwidth}
  \vspace{1pt}
    \textbf{C}  
\end{subfigure}
\begin{subfigure}[t]{0.9\textwidth}
  \vspace{0pt}
    \includegraphics[width=\textwidth]{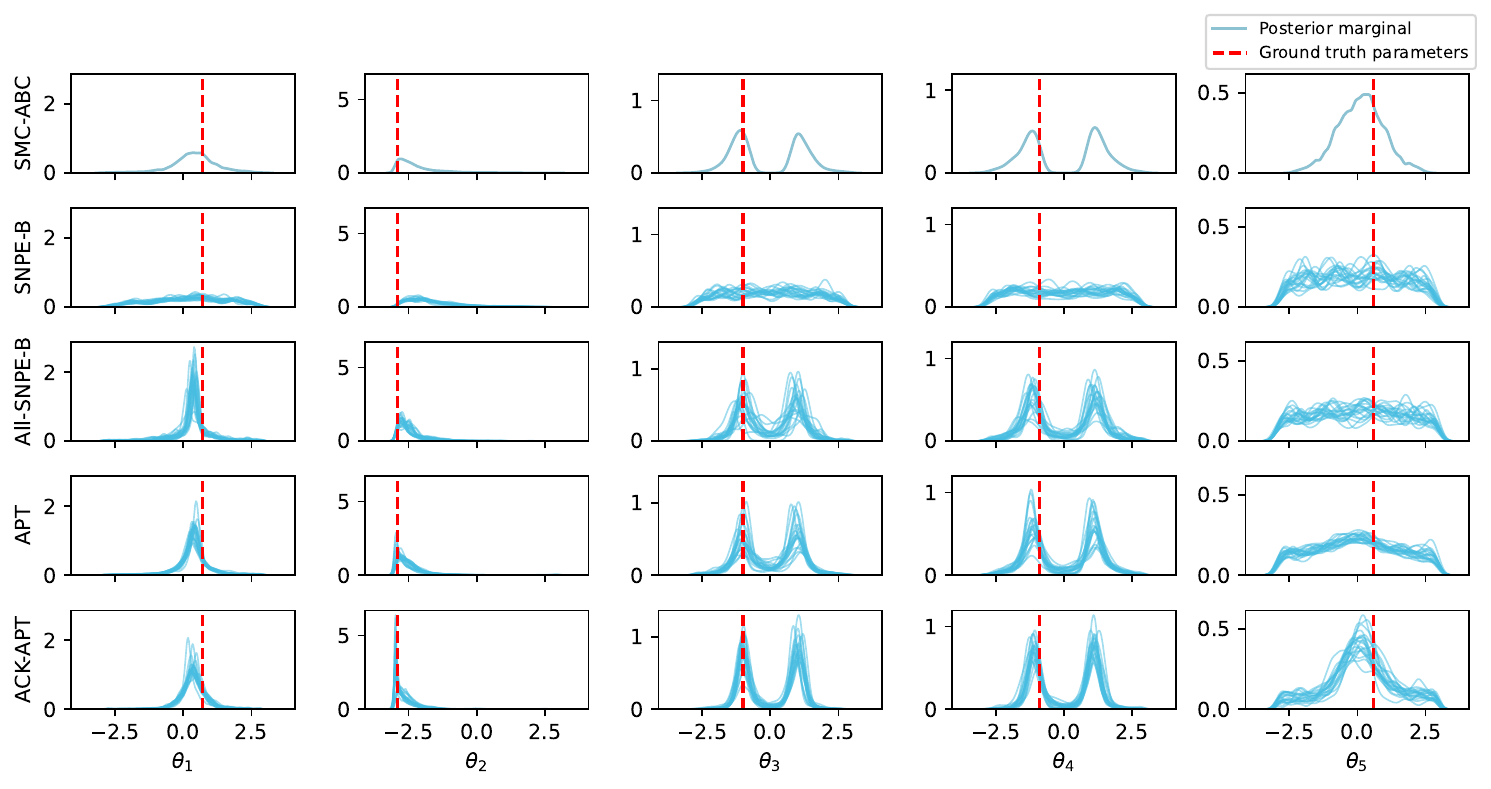}
\end{subfigure}

\hspace{0.02\textwidth}
\begin{subfigure}[t]{0.02\textwidth}
  \vspace{1pt}
    \textbf{D}  
\end{subfigure}
\begin{subfigure}[t]{0.9\textwidth}
  \vspace{0pt}
    \includegraphics[width=\textwidth]{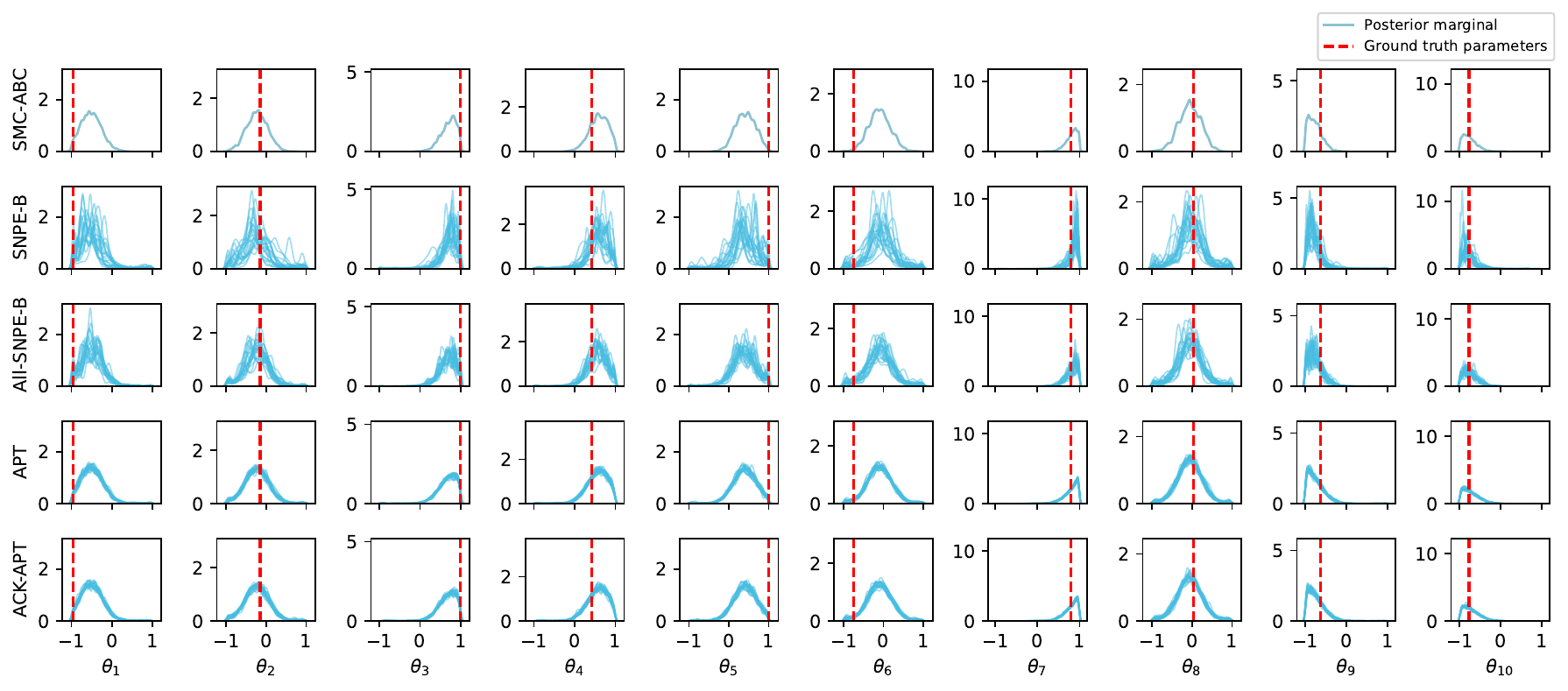}
\end{subfigure}

\caption{\textbf{Comparison of posterior densities evaluating at $x_o$. } 
\textbf{A.} M/G/1 model. 
\textbf{B.} Lotka-Volterra model. 
\textbf{C.} SLCP model. 
\textbf{D.} Gaussian linear model. 
Our proposed methods, All-SNPE-B and ACK-APT, demonstrate superior performance in accurately capturing the high-density regions of the posterior distributions, particularly for the SLCP and Lotka-Volterra models where the original SNPE-B method falls short.}
\label{fig:post_compare} 
\end{figure}

\end{document}